\pdfoutput=1

\documentclass{article}
\usepackage{kr}

\usepackage{times}
\usepackage{soul}
\usepackage{url}
\usepackage[hidelinks]{hyperref}
\usepackage[utf8]{inputenc}
\usepackage[small]{caption}
\usepackage{graphicx}
\usepackage{amsmath}
\usepackage{amssymb}
\usepackage{amsthm}
\usepackage{booktabs}
\usepackage{algorithm}
\usepackage{algorithmic}
\usepackage{wasysym}
\usepackage[title]{appendix}
\newtheorem{definition}{Definition}
\newtheorem{theorem}[definition]{Theorem}

\title{Expressing and Exploiting the Common Subgoal Structure \\
of Classical Planning Domains Using Sketches}
\author{
Dominik Drexler$^1$\and
Jendrik Seipp$^1$\and
Hector Geffner$^{3,2,1}$
\affiliations
$^1$Link{\"o}ping University, Link\"oping, Sweden\\
$^2$Universitat Pompeu Fabra, Barcelona, Spain\\
$^3$Instituci\'{o} Catalana de Recerca i Estudis Avan\c{c}ats (ICREA), Barcelona, Spain\\
\emails
\{dominik.drexler, jendrik.seipp\}@liu.se, hector.geffner@upf.edu
}

\newcommand{\remarktodo}[1]{\todo[color=yellow!30]{\upshape #1}}
\renewcommand{\remarktodo}[1]{}

\newcommand{\intextcite}[1]{\citeauthor{#1} (\citeyear{#1})}

\newcommand{\tup}[1]{(#1)}

\newcommand{\iw}[1]{\ensuremath{\text{IW}(#1)}\xspace}

\newcommand{\siw}{\ensuremath{\text{SIW}}\xspace}
\newcommand{\siwR}{\ensuremath{\text{SIW}_{\text R}}\xspace}
\newcommand{\siwRPhi}{\ensuremath{\text{SIW}_{\text R_\Phi}}\xspace}
\newcommand{\lama}{\ensuremath{\text{LAMA}}\xspace}
\newcommand{\bfws}{\ensuremath{\text{Dual-BFWS}}\xspace}

\newcommand{\Q}{\mathcal{Q}}

\newcommand{\sketch}{\ensuremath{R_\Phi}}

\newcommand{\pplus}{\hspace{-.05em}\raisebox{.15ex}{\footnotesize$\uparrow$}}
\newcommand{\mminus}{\hspace{-.05em}\raisebox{.15ex}{\footnotesize$\downarrow$}}

\newcommand{\EQ}[1]{#1{\,=\,}0}
\newcommand{\GT}[1]{#1{\,>\,}0}

\newcommand{\DEC}[1]{#1\mminus}
\newcommand{\INC}[1]{#1\pplus}
\newcommand{\UNK}[1]{#1?}

\newcommand{\prule}[2]{\{ #1 \} \mapsto \{ #2 \}}

\newcommand{\Omit}[1]{}

\newcommand{\domain}{\ensuremath{D}}
\newcommand{\instance}{\ensuremath{I}}
\newcommand{\goal}{\ensuremath{\mathit{Goal}}}
\newcommand{\initial}{\ensuremath{\mathit{Init}}}
\newcommand{\states}{\ensuremath{S}}
\newcommand{\initialstate}{\ensuremath{s_0}}
\newcommand{\goalstates}{\ensuremath{G}}
\newcommand{\actions}{\ensuremath{\mathit{Act}}}
\newcommand{\successor}{\ensuremath{f}}
\newcommand{\applicability}{\ensuremath{A}}

\newcommand{\bclevelone}{\ensuremath{\mathit{l}1}}
\newcommand{\bcleveltwo}{\ensuremath{\mathit{l}2}}
\newcommand{\bshand}{\ensuremath{\mathit{h}}}
\newcommand{\bsshot}{\ensuremath{\mathit{g}}}
\newcommand{\bsshaker}{\ensuremath{\mathit{t}}}
\newcommand{\bsingredient}{\ensuremath{\mathit{i}}}
\newcommand{\bsbeverage}{\ensuremath{\mathit{b}}}
\newcommand{\bscocktail}{\ensuremath{\mathit{c}}}

\newcommand{\bpontable}[1]{\ensuremath{\mathit{ontable}(#1)}}
\newcommand{\bpholding}[2]{\ensuremath{\mathit{holding}}(#1, #2)}
\newcommand{\bpempty}[1]{\ensuremath{\mathit{empty}(#1)}}
\newcommand{\bpcontains}[2]{\ensuremath{\mathit{contains}(#1, #2)}}
\newcommand{\bpclean}[1]{\ensuremath{\mathit{clean}(#1)}}
\newcommand{\bpused}[2]{\ensuremath{\mathit{used}(#1, #2)}}
\newcommand{\bpshakerlevel}[2]{\ensuremath{\mathit{shaker}\text{-}\mathit{level}(#1, #2)}}

\newcommand{\bfunserved}{\ensuremath{\mathit{g}}}
\newcommand{\bfused}{\ensuremath{\mathit{u}}}
\newcommand{\bfconsistentone}{\ensuremath{\mathit{c}_1}}
\newcommand{\bfconsistenttwo}{\ensuremath{\mathit{c}_2}}

\newcommand{\cckitchen}{\ensuremath{\mathit{kitchen}}}
\newcommand{\cssandwich}{\ensuremath{\mathit{s}}}
\newcommand{\cstray}{\ensuremath{\mathit{p}}}
\newcommand{\cstable}{\ensuremath{\mathit{t}}}
\newcommand{\cschild}{\ensuremath{\mathit{c}}}
\newcommand{\cpatkitchensandwich}[1]{\ensuremath{\mathit{at}\text{-}\mathit{kitchen}\text{-}\mathit{sandwich}(#1)}}
\newcommand{\cpnotexists}[1]{\ensuremath{\mathit{notexists}(#1)}}
\newcommand{\cpontray}[2]{\ensuremath{\mathit{ontray}(#1, #2)}}
\newcommand{\cpat}[2]{\ensuremath{\mathit{at}(#1, #2)}}
\newcommand{\cpserved}[1]{\ensuremath{\mathit{served}(#1)}}
\newcommand{\cpnoglutensandwich}[1]{\ensuremath{\mathit{no}\text{-}\mathit{gluten}\text{-}\mathit{sandwich}(#1)}}

\newcommand{\cfallergicchild}{\ensuremath{\mathit{c_g}}}
\newcommand{\cfregularchild}{\ensuremath{\mathit{c_r}}}
\newcommand{\cfglutenfreesandwichkitchen}{\ensuremath{\mathit{s_g^k}}}
\newcommand{\cfsandwichkitchen}{\ensuremath{\mathit{s^k}}}
\newcommand{\cfglutenfreesandwichtray}{\ensuremath{\mathit{s_g^t}}}
\newcommand{\cfsandwichtray}{\ensuremath{\mathit{s^t}}}

\newcommand{\dspackage}{\ensuremath{\mathit{p}}}
\newcommand{\dstruck}{\ensuremath{\mathit{t}}}
\newcommand{\dsdriver}{\ensuremath{\mathit{d}}}
\newcommand{\dslocation}{\ensuremath{\mathit{c}}}  
\newcommand{\dpin}[2]{\ensuremath{\mathit{in}(#1, #2)}}
\newcommand{\dpat}[2]{\ensuremath{\mathit{at}(#1, #2)}}
\newcommand{\dpdriving}[2]{\ensuremath{\mathit{driving}(#1, #2)}}
\newcommand{\dfpackages}{\ensuremath{\mathit{p}}}
\newcommand{\dftrucks}{\ensuremath{\mathit{t}}}
\newcommand{\dfdrivergoal}{\ensuremath{\mathit{d}_g}}
\newcommand{\dfdrivertruck}{\ensuremath{\mathit{d}_t}}
\newcommand{\dfboarded}{\ensuremath{\mathit{b}}}
\newcommand{\dfloaded}{\ensuremath{\mathit{l}}}

\newcommand{\fcblack}{\ensuremath{\mathit{black}}}
\newcommand{\fcwhite}{\ensuremath{\mathit{white}}}
\newcommand{\fstile}{\ensuremath{\mathit{t}}}
\newcommand{\fscolor}{\ensuremath{\mathit{c}}}
\newcommand{\fsrobot}{\ensuremath{\mathit{a}}}
\newcommand{\fppainted}[2]{\ensuremath{\mathit{painted}(#1, #2)}}
\newcommand{\fprobotat}[2]{\ensuremath{\mathit{robot}\text{-}\mathit{at}(#1, #2)}}
\newcommand{\fprobothas}[2]{\ensuremath{\mathit{robot}\text{-}\mathit{has}(#1, #2)}}
\newcommand{\fpclear}[1]{\ensuremath{\mathit{clear}(#1)}}
\newcommand{\fpup}[2]{\ensuremath{\mathit{up}(#1, #2)}}
\newcommand{\fpdown}[2]{\ensuremath{\mathit{down}(#1, #2)}}
\newcommand{\ffinvariant}{\ensuremath{v}}
\newcommand{\ffunpaintedtiles}{\ensuremath{g}}

\newcommand{\gscell}{\ensuremath{\mathit{c}}}
\newcommand{\gslock}{\ensuremath{\mathit{d}}}
\newcommand{\gskey}{\ensuremath{\mathit{e}}}
\newcommand{\gpatrobot}[1]{\ensuremath{\mathit{at}\text{-}\mathit{robot}(#1)}}
\newcommand{\gpholding}[1]{\ensuremath{\mathit{holding}(#1)}}
\newcommand{\gpopen}[1]{\ensuremath{\mathit{open}(#1)}}
\newcommand{\gpat}[2]{\ensuremath{\mathit{at}(#1, #2)}}
\newcommand{\gflocked}{\ensuremath{\mathit{l}}}
\newcommand{\gfkey}{\ensuremath{\mathit{k}}}
\newcommand{\gfopen}{\ensuremath{\mathit{o}}}
\newcommand{\gftarget}{\ensuremath{\mathit{t}}}

\newcommand{\schot}{\ensuremath{\mathit{hot}}}

\newcommand{\sstemperature}{\ensuremath{\mathit{t}}}
\newcommand{\ssobject}{\ensuremath{\mathit{a}}}
\newcommand{\spshape}[2]{\ensuremath{\mathit{shape}(#1, #2)}}
\newcommand{\spsurfacecondition}[2]{\ensuremath{\mathit{surface}\text{-}\mathit{condition}(#1, #2)}}
\newcommand{\sppainted}[2]{\ensuremath{\mathit{painted}(#1, #2)}}
\newcommand{\spscheduled}[1]{\ensuremath{\mathit{scheduled}}(#1)}
\newcommand{\spnotscheduled}[1]{\ensuremath{\neg\mathit{scheduled}}(#1)}
\newcommand{\sptemperature}[2]{\ensuremath{\mathit{temperature}}(#1, #2)}
\newcommand{\sfshape}{\ensuremath{\mathit{p}_1}}
\newcommand{\sfsurface}{\ensuremath{\mathit{p}_2}}
\newcommand{\sfcolor}{\ensuremath{\mathit{p}_3}}
\newcommand{\sfhot}{\ensuremath{\mathit{h}}}
\newcommand{\sfoccupied}{\ensuremath{\mathit{o}}}

\newcommand{\tsgood}{\ensuremath{\mathit{g}}}
\newcommand{\tsplace}{\ensuremath{\mathit{p}}}
\newcommand{\tstruck}{\ensuremath{\mathit{t}}}

\newcommand{\tpat}[2]{\ensuremath{\mathit{at}(#1, #2)}}
\newcommand{\tpreadytoload}[3]{\ensuremath{\mathit{ready}\text{-}\mathit{to}\text{-}\mathit{load}(#1, #2, #3)}}
\newcommand{\tploaded}[3]{\ensuremath{\mathit{loaded}(#1, #2, #3)}}
\newcommand{\tpstored}[2]{\ensuremath{\mathit{stored}(#1, #2)}}
\newcommand{\tponsale}[3]{\ensuremath{\mathit{on}\text{-}\mathit{sale}(#1, #2, #3)}}
\newcommand{\tfloaded}{\ensuremath{\mathit{u}}}
\newcommand{\tfstored}{\ensuremath{\mathit{w}}}

\newcommand{\emptyfeature}[1]{\ensuremath{\mathit{Empty}(#1)}}
\newcommand{\countfeature}[1]{\ensuremath{\mathit{Count}(#1)}}
\newcommand{\conceptdistancefeature}[3]{\ensuremath{\mathit{ConceptDist}(#1, #2, #3)}}

\newcommand{\roledistancefeature}[3]{\ensuremath{\mathit{RoleDist}(#1, #2, #3)}}
\newcommand{\sumroledistancefeature}[3]{\ensuremath{\mathit{SumRoleDist}(#1, #2, #3)}}

\usepackage{booktabs}
\usepackage{xspace}
\usepackage{todonotes}
\presetkeys{todonotes}{inline}{}

\begin{document}

\maketitle
\begin{abstract}
Width-based planning methods deal with conjunctive goals by decomposing problems into
subproblems of low width. Algorithms like \siw thus fail when the goal is not easily
serializable in this way or when some of the subproblems have a high width.
In this work, we address these limitations by using a simple but powerful
language for expressing finer problem decompositions introduced recently by Bonet and Geffner,
called \emph{policy sketches}. A policy sketch $R$ over a set of Boolean and numerical
features is a set of sketch rules $C \mapsto E$ that express how the values of these features
are supposed to change. Like general policies, policy sketches are domain general,
but unlike policies, the changes captured by sketch rules do not need to
be achieved in a single step. We show that many planning domains that cannot be solved by \siw
are provably solvable in low polynomial time with the \siwR algorithm,
the version of SIW that employs user-provided policy sketches. Policy
sketches are thus shown to be a powerful language for expressing domain-specific
knowledge in a simple and compact way and a convenient alternative
to languages such as HTNs or temporal logics. Furthermore, they make it easy
to express general problem decompositions and prove key properties of them
like their width and complexity.
\end{abstract}

\section{Introduction}

The success of width-based methods in classical planning is the result of two main ideas:
the use of conjunctive goals for decomposing a problem into subproblems,
and the observation that the width of the subproblems is often bounded and small \cite{lipovetzky-geffner-ecai2012}.
When these assumptions do not hold, pure width-based methods struggle and
need to be extended with heuristic estimators or landmark counters that yield finer problem decompositions
\cite{lipovetzky-et-al-aaai2017,lipovetzky-geffner-icaps2017}. These hybrid approaches have resulted in state-of-the-art planners
but run into shortcomings of their own: unlike pure width-based search methods,
they require declarative, PDDL-like action models and thus cannot plan with black box simulators
\cite{lipovetzky-et-al-ijcai2015,shleyfman-et-al-ijcai2016,geffner-geffner-aiide2015}, and they produce decompositions that are ad-hoc
and difficult to understand. Variations of these approaches, where the use of declarative action models
is replaced by  polynomial searches, have pushed the scope of pure-width based search methods
\cite{frances-et-al-ijcai2017}, but they do not fully overcome  their basic limits:
\emph{top goals that are not easily serializable or that have a high width.}
These are indeed the limitations of one of the simplest width-based search methods,
Serialized Iterated Width (SIW) that greedily achieves top goal first,
one at a time, using IW searches \cite{lipovetzky-geffner-ecai2012}.

In this work, we address the limitations of the SIW algorithm differently by using
a simple but powerful language for expressing richer problem decompositions
recently introduced by Bonet and Geffner \shortcite{bonet-geffner-aaai2021}, called \textbf{policy sketches}.
A policy sketch is a set of sketch rules over a set of Boolean and numerical features
of the form $C \mapsto E$ that express
how the values of the features are supposed to change.
Like \textbf{general policies} \cite{bonet-geffner-ijcai2018},
sketches are general and not tailored to specific instances
of a domain, but unlike policies, the feature changes expressed by sketch rules
represent {subgoals} that do not need to be achieved in a single step.

We pick up here where Bonet and Geffner left off and show
that many benchmark planning domains that SIW cannot solve
are provably solvable in low polynomial time through the \siwR algorithm,
the version of SIW that makes use of user-provided policy sketches.
Policy sketches are thus shown to be a powerful \textbf{language for expressing domain-specific
knowledge} in a simple and compact way and a convenient alternative to languages such as HTNs or temporal logics.
Bonet and Geffner introduce the language of sketches and the theory behind them; we show their use and the properties that follow from them.
As we will see, unlike HTNs and temporal logics, sketches can be used to \textbf{express and exploit the common subgoal structure
of planning domains} without having to express how subgoals are to be achieved.
Also, by being simple and succinct they provide a convenient target language
for learning the subgoal structure of domains automatically,
although this problem, related to the problem of learning general policies \cite{bonet-et-al-aaai2019,frances-et-al-aaai2021},
is outside the scope of this paper.
In this work, we use sketches to solve domains in polynomial time, which excludes intractable domains.
Indeed, intractable domains do not have general policies nor sketches of bounded width
and require non-polynomial searches.
Sketches and general policies, however, are closely related: sketches provide
the skeleton of a general policy, or a general policy with ``holes'' that
are filled by searches that can be shown to be polynomial \cite{bonet-geffner-aaai2021}.

The paper is organized as follows.
We review the notions of width, sketch width, and policy sketches
following Bonet and Geffner \shortcite{bonet-geffner-aaai2021}.
We show then that it is possible to write compact and transparent policy
sketches for many domains, establish their widths, and analyze the performance
of the \siwR algorithm. We compare sketches to HTNs and temporal logics,
briefly discuss the challenge of learning sketches automatically,
and summarize the main contributions.

\section{Planning and Width}

A \emph{classical planning problem} or instance  $P=\tup{\domain,\instance}$
is assumed to be given by a first-order domain $\domain$ with action schemas defined over some
domain predicates, and instance information $\instance$ describing a set of objects,
and two sets of ground literals describing the initial situation $\initial$ and goal description $\goal$.
The initial situation is assumed to be complete such
that either $L$ or its complement is in $\initial$.
A problem $P$ defines a state model $S(P)=\tup{\states,\initialstate,\goalstates,\actions,\applicability,\successor}$ where the states in $S$
are the truth valuations over the ground atoms represented by the set of  literals that they make true, the initial state $s_0$ is $\initial$,
the set of goal states $\goalstates$ includes all of those that make the goal atoms in $\goal$ true, the actions $Act$ are the ground actions obtained from
the action schemas and the objects, the actions $\applicability(s)$ applicable in state $s$ are those whose preconditions are true in $s$,
and the state transition function $f$ maps a state $s$ and an action $a \in \applicability(s)$ into the successor state $s'=f(a,s)$.
A \emph{plan} $\pi$ for $P$ is a sequence of actions $a_0,\ldots,a_n$ that is executable in $s_0$ and maps the initial state $s_0$
into a goal state; i.e., $a_i \in \applicability(s_i)$, $s_{i+1}=f(a_i,s_i)$, and $s_{n+1}\in\goalstates$.
The cost of a plan is assumed to be given by its length, and a plan is optimal if there is no shorter plan.
We'll be interested in solving \emph{collections} of well-formed instances $P=\tup{\domain,\instance}$ over fixed domains $\domain$
denoted as $\Q_\domain$ or simply as $\Q$.

The most basic width-based search method for solving a planning problem $P$  is \iw{1}.
It performs a standard breadth-first search in the rooted directed graph associated with the state model
$S(P)$ with one modification: \iw{1} prunes a newly generated state if it does not make
an atom $r$ true for the first time in the search. The procedure \iw{k} for $k > 1$ is like \iw{1}
but prunes a state if a newly generated state does not make a collection of up to $k$ atoms true for the first time.
Underlying the IW algorithms is the notion of \emph{problem width} \cite{lipovetzky-geffner-ecai2012}:
\begin{definition}[Width]
  \label{def:width}
  Let $P$ be a classical planning problem with initial state $\initialstate$ and goal states $\goalstates$.
  The \textbf{width} $w(P)$ of $P$ is the minimum $k$ for which there exists a
  \textbf{sequence $t_0,t_1,\ldots,t_m$ of atom tuples $t_i$},
  each consisting of at most $k$ atoms, such that:
  \begin{enumerate}
    \item $t_0$ is true in initial state $\initialstate$ of $P$,
    \item any optimal plan for $t_i$ can be extended into an optimal plan for $t_{i+1}$ by adding a single action, $i=1,\ldots,n-1$,
    \item if $\pi$ is an optimal plan for $t_m$, then $\pi$ is an optimal plan for $P$.
  \end{enumerate}
\end{definition}

If a problem $P$ is unsolvable, $w(P)$ is set to the number of variables
in $P$, and if $P$ is solvable in at most one step, $w(P)$ is set to $0$
\cite{bonet-geffner-aaai2021}. Chains of tuples $\theta = (t_0, t_1, \dots, t_m)$
that comply with conditions 1--3 are called \textbf{admissible}, and the
size of $\theta$ is the size $|t_i|$ of the largest tuple in the chain.
We talk about the third condition by saying that $t_m$ implies $G$ in the admissible chain $t_1,t_2,\ldots,t_m$.
The width $w(P)$ is thus $k$ if $k$ is the minimum size of an admissible chain for $P$.
If the width of a problem $P$ is $w(P) = k$,
\iw{k} finds an optimal (shortest) plan for $P$ in time and space
that are exponential in $k$ and not in the number of problem variables $N$
as breadth-first search.

The \iw{k} algorithm expands up to $N^k$ nodes, generates up to $bN^k$ nodes, and runs in time and space
$O(bN^{2k-1})$ and $O(bN^k)$, respectively,  where $N$ is the number of atoms and $b$ bounds the branching factor in problem $P$.
\iw{k} is guaranteed to solve $P$ optimally (shortest path) if $w(P) \leq k$.
If the width of $P$ is not known, the \textbf{IW}  algorithm can be run instead which calls \iw{k} iteratively for
$k\!=\!0,1,\ldots,N$ until the problem is solved, or found to be unsolvable.

While IW and \iw{k} algorithms  are not practical by themselves, they are building
blocks for other methods. Serialized Iterated Width or \textbf{SIW} \cite{lipovetzky-geffner-ecai2012},
starts  at the initial  state $s=s_0$ of  $P$,  and then performs  an IW search
from $s$ to find  a shortest path to state $s'$ such that
$\#g(s') < \#g(s)$ where $\#g$ counts the number of top goals of $P$
that are not true in $s$. If $s'$ is not a goal state, $s$ is set to $s'$
and the loop repeats until a goal state is reached.

In practice, the \iw{k} searches in SIW are  limited to $k \leq 2$ or $k \leq 3$, so that SIW  solves
a problem or fails in low polynomial time. SIW performs  well in many benchmark domains but
fails in problems where the width of some  top goal is not small, or the top
goals can't be serialized greedily. More recent methods
address these limitations by using width-based notions (novelty measures) in complete best-first search algorithms
\cite{lipovetzky-et-al-aaai2017,frances-et-al-ijcai2017}, yet they also struggle in problems where some top goals have high width.
In this work, we take a different route: we keep the greedy polynomial searches
underlying SIW but consider a richer class of problem decompositions expressed through sketches.
The resulting planner \siwR is \textbf{not} domain-independent like SIW,
but it illustrates that a bit of domain knowledge can go a long way in the effective solution of arbitrary domain instances.

\section{Features and Sketches}

A \textbf{feature} is a function of the state over a class of problems $\Q$.
The features considered in the language of sketches are {Boolean},
taking values in the Boolean domain, or {numerical},
taking values in the non-negative integers. For a set $\Phi$ of features
and a state $s$ of some instance $P$ in $\Q$, $f(s)$ is the \textbf{feature valuation}
determined by a state $s$. A {Boolean}  feature valuation over $\Phi$
refers instead to the valuation of the expressions $p$ and $n=0$
for Boolean and numerical features $p$ and $n$ in $\Phi$. If $f$
is a feature valuation, $b(f)$  will denote the Boolean feature valuation
determined by $f$ where the values of numerical features are just compared with $0$.

The set of features $\Phi$  \textbf{distinguish} or \textbf{separate
the goals}  in $\Q$ if there is a set $B_\Q$ of Boolean feature valuations
such that $s$ is a goal state of an instance $P \in \Q$ iff $b(f(s)) \in B_\Q$.
For example, if $\Q_{clear}$ is the set of all blocks world instances with stack/unstack operators
and common goal $clear(x) \land handempty$ for some block $x$, and $\Phi=\{n(x),H\}$
are the features that track  the number of blocks above $x$ and whether
the gripper is holding a block, then there is a single Boolean goal valuation that
makes the expression $n(x)=0$ true and $H$ false.

A \textbf{sketch rule} over features $\Phi$ has the form $C \mapsto E$
where $C$ consists of Boolean feature conditions, and $E$ consists of feature effects.
A Boolean (feature) condition is of the form $p$ or $\neg p$ for a Boolean feature $p$ in $\Phi$,
or $n=0$ or $n>0$ for a numerical feature $n$ in $\Phi$.
A feature   effect is an expression of the form $p$, $\neg p$,
or $\UNK{p}$ for a Boolean feature $p$ in $\Phi$, and $\DEC{n}$, $\INC{n}$,
or $\UNK{n}$ for a numerical feature $n$ in $\Phi$.
The syntax of sketch rules is the syntax of the policy rules
used to define generalized policies \cite{bonet-geffner-ijcai2018},
but their semantics is different.
In policy rules, the effects have to be delivered in one step by state transitions,
while in sketch rules, they can be delivered by longer state sequences.

A \textbf{policy sketch}  $R_\Phi$ is a collection of sketch rules over the features $\Phi$
and the sketch is {well-formed} if it is built from features that \textbf{distinguish the goals} in $\Q$,
and is \textbf{terminating} (to be made precise below).  A \textbf{well-formed sketch} for a class of problems $\Q$
defines a {serialization} over $\Q$; namely, a ``preference'' ordering `$\prec$' over the feature valuations that is
irreflexive and transitive, and which is given by the smallest strict partial order that satisfies $f' \prec f$ if $f'$ is not a goal valuation and
the pair of feature valuations $(f,f')$ \textbf{satisfies a sketch rule}  $C \mapsto E$.
This happens when: 1)~$C$ is true in $f$, 2)~the Boolean effects $p$ ($\neg p$) in $E$ are true in $f'$,
3)~the numerical effects  are satisfied by the pair $(f,f')$; i.e.,
if  $\DEC{n}$  in $E$ (resp. $\INC{n}$), then value of $n$ in $f'$ is smaller than in $f$, i.e.,
$f'_n < f_n$  (resp. $f_n > f'_n$), and 4)~Features that do not occur in $E$ have the same value in $f$ and $f'$.
Effects $\UNK{p}$ and $\UNK{n}$ do not constraint the value of the features $p$ and $n$ in any way,
and by including them in $E$, we say that they can change in any way, as opposed to features
that are not mentioned in $E$ whose values in $f$ and $f'$ must be the same.

Following Bonet and Geffner, we do not use the serializations determined by sketches
but their associated problem \textbf{decompositions}.
The set of \textbf{subgoal states} $G_r(s)$ associated with a sketch rule $r: C \mapsto E$ in $R_\Phi$
and a state $s$ for a problem $P \in Q$, is empty if $C$ is not true in $f(s)$,
and else is given by the set of states $s'$ with feature valuations $f(s')$ such that
the pair $(f,f')$ for $f=f(s)$ and $f'=f(s')$ satisfies the sketch rule $r$.
Intuitively, when in a state $s$, the subgoal states $s'$ in $G_r(s)$
provide a stepping stone in the search for plans connecting $s$ to the goal of $P$.


\section{Serialized Iterated Width with Sketches}

The \textbf{\siwR}  algorithm is a variant of the SIW algorithm that uses a given
sketch $R=R_{\Phi}$ for solving problems $P \in \Q$.
\siwR starts at the state $s:=s_0$, where $s_0$ is the initial state of  $P$,
and then performs an IW search to find a state $s'$
that is closest from $s$ such that $s'$ is a goal state of $P$ or
a subgoal state in $G_r(s)$ for some sketch rule $r$ in $R$.
If $s'$ is not a goal state, then $s$ is set to $s'$, $s := s'$,
and the loop repeats until a goal state is reached.
The features  define subgoal states through the  sketch rules
but otherwise play no role in the IW searches.

The only difference between \siw and \siwR is that in \siw\
each IW search finishes  when the goal counter $\#g$ is decremented,
while in \siwR, when a  goal or subgoal state is reached.
The behavior of plain SIW can be emulated in \siwR using the
single sketch rule $\prule{\GT{\#g}}{\DEC{\#g}}$ in $R$ when the goal counter $\#g$ is the only
feature, and the rule $\prule{\GT{\#g}}{\DEC{\#g},\UNK{p},\UNK{n}}$, when $p$
and $n$ are the other features. This last rule says that it is always ``good'' to decrease
the goal counter independently of the effects on other features, or alternatively,
that decreasing the goal counter is  a subgoal from any state $s$ where $\#g(s)$ is positive.

The complexity of \siwR\ over a class of problems $\Q$
can be  bounded in terms of the \textbf{width of the sketch} $R_\Phi$,
which is given by the width of the possible subproblems that can be encountered
during the execution of \siwR when solving a problem  $P$ in $\Q$.
For this, let us define the set $S_{R}(P)$  of reachable states in $P$
when following the sketch $R=R_\Phi$ recursively as follows:
1)~the initial state $s$ of $P$ is in $S_{R}(P)$,
2)~the (subgoal) states $s'\in G_r(s)$ that are closest to $s$ are in $S_R(P)$ if $s\in S_R(P)$ and $r \in R$.
The states in $S_R(P)$ are called the $R$-reachable states in $P$.
The width of the sketch $R$ is then \cite{bonet-geffner-aaai2021}:

\begin{definition}[Sketch width]
    \label{def:sketch_width}
    The width of the sketch $R=R_\Phi$ at state $s$ of problem $P \in \Q$, denoted  $w_{R}(P[s])$,
    is the width $k$ of the subproblem $P[s]$ that is like $P$ but with initial state $s$
    and goal states that contain those of $P$ and those in $G_r(s)$ for all $r \in R$.
     The \textbf{width of the sketch} $R$ over $\Q$ is $w_{R}(\Q) = \max_{P,s} w_{R}(P[s])$
     for $P \in \Q$ and $s\in S_{R}(P)$.\footnote{This definition changes the one by Bonet and Geffner slightly
       by restricting the reachable states $s$ to those that are $R$-reachable; i.e., part of $S_R(P)$.
       This distinction is convenient when $\Q$ does not contain all possible ``legal'' instances $P$ but only
       those in which the initial situations complies with certain conditions (e.g., robot arm is empty).
       In those cases, the sketches for $\Q$ do not have to cover all reachable states.}
\end{definition}

The time complexity of \siwR can then be expressed as follows, under the assumption
that the features are all linear \cite{bonet-geffner-aaai2021}:

\begin{theorem}
    \label{theorem:siwr_complexity}
    If width $w_{R}(\Q)=k$, \siwRPhi solves any $P\in\Q$ in $O(bN^{|\Phi|+2k-1})$ time and $O(bN^k+N^{|\Phi|+k})$ space.
\end{theorem}

A feature is linear if it can be computed in linear time and can take a linear number of values at most.
In both cases, the linearity is in the number of atoms $N$ in the problem $P$ in $\Q$.
If the features are not linear but polynomial in $P$, the bounds on \siwR\ remain polynomial as well
(both $k$ and $\Phi$ are constants).

Bonet and Geffner introduce and study the language of sketches as a variation
of the language of general policies and their relation to the width and serialized width of planning domains.
They illustrate the use of sketches in a simple example (Delivery) but focus mainly on the theoretical aspects.
Here we focus instead on their use for modeling domain-specific knowledge in the standard planning benchmarks
as an alternative to languages like HTNs.


\section{Sketches for Classical Planning Domains}

In this section, we present policy sketches for a representative set of
classical planning domains from the benchmark set of the International
Planning Competition (IPC). All of the chosen domains are solvable
suboptimally in polynomial time, but plain \siw fails to solve them. There
are two main reasons why \siw fails. First, if achieving a single goal
atom already has a sufficiently large width. Last, greedy goal serialization
generates such avoidable subproblems, including reaching unsolvable states.

We provide a handcrafted sketch for each of the domains
and show that it is well-formed and has small sketch width.
These sketches allow \siwR to solve all instances of the domain in low polynomial time and space by
Theorem~\ref{theorem:siwr_complexity}. Furthermore, we impose a low
polynomial complexity bound on each feature, i.e., at most quadratic in
the number of grounded atoms. Such a limitation is necessary since
otherwise, we could use a numerical feature that encodes the optimal value function $V^*(s)$, i.e.,
the perfect goal distance of all states $s$. With such a feature,
the sketch rule $\prule{V^* > 0}{\DEC{V^*}}$ makes all problems trivially solvable.
Even with linear and quadratic features, we
can capture complex state properties such as distances between objects.

\subsection{Proving Termination and Sketch Width}

For each sketch introduced below we show that it uses goal-separating
features, is terminating and has bounded and small sketch width. Showing
that the features are goal separating is usually direct.



Proving \textbf{termination} is required to ensure that by iteratively
moving from a state $s$ to a subgoal state $s' \in G_r(s)$ we cannot get
trapped in a cycle. The conditions under which a sketch $R_\Phi$ is
terminating are similar to those that ensure that a general policy
$\pi_\Phi$ is terminating \cite{srivastava-et-al-aaai2011,bonet-geffner-jair2020,bonet-geffner-aaai2021},
and can be determined in polynomial time in the size of the sketch graph
$G(\pi_\Phi)$ using the \textsc{Sieve} procedure
\cite{srivastava-et-al-aaai2011,bonet-geffner-jair2020}.
Often, however, a simple syntactic procedure suffices
that eliminates sketch rules, one after the other until none is left.
This syntactic procedure is sound but not complete in general. In the
following, we say that a rule $C\mapsto E$ \emph{changes} a Boolean
feature $b$ if $b \in C$ and $\neg b \in E$ or the other way around. The
procedure iteratively applies one of the following cases until no rule is
left (the sketch terminates) or until no further cases apply (there may be
an infinite loop in the sketch): (a) eliminate a rule if it decreases a
numerical feature $n$ ($\DEC{n}$) that no other remaining rule can
increase ($\INC{n}$ or $\UNK{n}$); (b) eliminate a rule $r$ if it changes
the value of a Boolean feature that no other remaining rule changes in the
opposite direction; (c) mark all features that were used for eliminating a
rule in (a) or (b) as these can only change finitely often; (d) remove
rules $C\mapsto E$ that decrease a numerical feature $n$ or that change
a Boolean feature $b$ to true (false) such that for all other remaining
rules $C'\mapsto E'$ it holds that if $E'$ changes the feature in the
opposite direction, i.e., $\INC{n}$, $\UNK{n}$ or changes $b$ to false
(true), there must be a condition on a variable in $C$ that is marked and
is complementary to the one in $C'$, e.g., $\GT{n} \in C$ and $\EQ{n} \in
C'$ or $b \in C$ and $\neg b \in C'$.

Showing that a sketch for problem class $\Q$ has \textbf{sketch width} $k$
requires to prove that for all $R$-reachable states $s$ in all problem
instances $P\in\Q$, the width of $P[s]$ is bounded by $k$. Remember that
$P[s]$ is like $P$ but with initial state $s$, and goal states $G$ of $P$
combined with goal states $G_r(s)$ of all $r\in R$.
The definition of $R$-reachability shows that we need a recursive proof
strategy: informally, we show that (1) the feature conditions of a rule $r$ with nonempty subgoal $G_r(s)$ are true in all
initial states $s$ of $\Q$, and (2) by following a rule, we land in a goal
state or a state $s'$ where the feature conditions of another rule $r'$ with nonempty subgoal $G_{r'}(s')$ are true. To show that the sketch
has width $k$, we prove that all subtasks $P[s]$ for traversed states $s$
have width $k$. We do this by providing an admissible chain
$t_1,\ldots,t_m$ of size at most $k$ where all optimal plans for $t_m$ are
also optimal plans for $P[s]$.
We overapproximate the set of $R$-reachable states where necessary to make the proofs more compact.
This implies that our results provide an upper bound on the actual sketch width but are sufficiently small.
For space reasons, we give only two exemplary proofs in the paper and present the remaining proofs in \intextcite{drexler-et-al-arxiv2021}.

\subsection{Floortile}

In the Floortile domain \cite{linareslopez-et-al-aij2015}, a set of robots have to paint the tiles of a rectangular grid.
There can be at most one robot $\fsrobot$ on each tile $\fstile$ at any time
and the predicate $\fprobotat{\fsrobot}{\fstile}$ is true iff $\fsrobot$ is on tile $\fstile$.
If there is no robot on tile $\fstile$ then $\fstile$ is marked as clear, i.e., $\fpclear{\fstile}$ holds.
Robots can move left, right, up or down, if the target tile is clear.
Each robot $\fsrobot$ is equipped with a brush that is configured
to either paint in $\fcblack$ or $\fcwhite$, e.g., $\fprobothas{\fsrobot}{\fcblack}$
is true iff the brush of robot $\fsrobot$ is configured to paint in $\fcblack$.
It is possible to change the color infinitely often.
The goal is to paint a rectangular subset of the grid in chessboard style.
If a tile $\fstile$ has color $\fscolor$ then the predicate $\fppainted{\fstile}{\fscolor}$
holds and additionally the tile is marked as not clear, i.e., $\fpclear{\fstile}$ does not hold.
A robot $\fsrobot$ can only paint tile $\fstile$
if $\fsrobot$ is on a tile $\fstile'$ that is below or above $\fstile$, i.e.,
$\fprobotat{\fsrobot}{\fstile'}$ holds, and $\fpup{\fstile'}{\fstile}$ or $\fpdown{\fstile'}{\fstile}$ holds.

Consider the set of features $\Phi = \{\ffunpaintedtiles, \ffinvariant\}$
where $\ffunpaintedtiles$ counts the number of unpainted tiles that need
to be painted and $\ffinvariant$ represents that the following condition
is satisfied: for each tile $t_1$ that remains to be painted
there exists a sequence of tiles
$t_1,\ldots,t_n$ such that each $t_i$ with $i=1,\ldots,n-1$ remains to be painted,
$t_n$ does not need to be painted, and
for all pairs $t_{i-1}, t_i$ with $i=2,\ldots,n$ holds that $t_i$ is above
$t_{i-1}$, i.e., $\fpup{t_{i-1}}{t_i}$, or for all pairs $t_{i-1}, t_i$ with $i=2,\ldots,n$
it holds that $t_i$ is below $t_{i-1}$, i.e., $\fpdown{t_{i-1}}{t_i}$.
Intuitively, $\ffinvariant$ is true iff a given state is solvable.
The set of sketch rules $R_\Phi$ contains the single rule
\begin{align*}
    r = \prule{\ffinvariant, \GT{\ffunpaintedtiles}}{\DEC{\ffunpaintedtiles}}
\end{align*}
which says that painting a tile such that the invariant $\ffinvariant$ remains satisfied is good.

\begin{theorem}
    The sketch for the Floortile domain is well-formed and has width $2$.
\end{theorem}
\begin{proof}
    Recall that a sketch is well-formed if it uses goal-separating
    features and is terminating.
    The features $\Phi$ are goal separating
    because the feature valuation $\EQ{\ffunpaintedtiles}$ holds in state $s$ iff $s$ is a goal state.
    The sketch $\sketch$ is terminating because $r$ decreases the numerical
    feature $\ffunpaintedtiles$ and no other rule increases $\ffunpaintedtiles$.

    It remains to show that the sketch width is $2$.
    Consider a Floortile instance $P$ with states $S$.
    If the initial state $s$ is a solvable non-goal state, then the feature conditions of $r$ are true, and the subgoal $G_r(s)$ is nonempty.
    If we reach such a subgoal state, then either the feature conditions of $r$ remain true
    because the invariant remains satisfied or the overall goal was reached.
    Next, we show that $P[s]$ with subgoal $G_r(s)$ in $R$-reachable state $s$ has width $2$.
    Consider states $S_1\subseteq S$ where the feature conditions of rule $r$ are true, i.e.,
    solvable states where a tile $\fstile$ must be painted in a color $\fscolor$.
    We do a three-way case distinction over states $S_1$.

    First, consider states $S_1^1\subseteq S_1$
    where some robot $\fsrobot$ on tile $\fstile_1$ that is configured to color $\fscolor$,
    can move to tile $\fstile_n$ above or below $\fstile$ to paint it.
    The singleton tuple $\fppainted{t}{\fscolor}$ implies $G_r(s)$ in $s\in S_1^1$
    in the admissible chain that consists of moving $\fsrobot$ from $\fstile_1$ to $\fstile_n$,
    while decreasing the distance to $\fstile_n$ in each step, and painting $\fstile$, i.e.,
    \begin{align*}
        (\fprobotat{\fsrobot}{\fstile_1},\ldots,\fprobotat{\fsrobot}{\fstile_n}, \fppainted{\fstile}{\fscolor}).
    \end{align*}
    Second, consider states $S_1^2\subseteq S_1$
    where the robot $\fsrobot$ must reconfigure its color from $\fscolor'$ to $\fscolor$ before painting.
    The tuple $(\fprobotat{\fsrobot}{\fstile_n}, \fppainted{\fstile}{\fscolor})$
    implies $G_r(s)$ in $s\in S_1^2$ in the admissible chain that consists of reconfiguring the color,
    and then moving closer and painting as before, i.e.,
    \begin{align*}
        (&(\fprobotat{\fsrobot}{\fstile_1}, \fprobothas{\fsrobot}{\fscolor'}), \\
         &(\fprobotat{\fsrobot}{\fstile_1}, \fprobothas{\fsrobot}{\fscolor}), \ldots, \\
         &(\fprobotat{\fsrobot}{\fstile_n}, \fprobothas{\fsrobot}{\fscolor}), \\
         &(\fprobotat{\fsrobot}{\fstile_n}, \fppainted{\fsrobot}{\fscolor})).
    \end{align*}
    We observe that reconfiguring requires an admissible chain
    of size $2$ because of serializing the reconfiguring and the moving part.
    Therefore, in the following case, we assume that the robot must reconfigure its color.

    Third, consider states $S_1^3\subseteq S_1$
    where robot $\fsrobot$ is standing on $\fstile$
    and there is a sequence of robots $\fsrobot_1,\ldots,\fsrobot_n$
    such that $\fsrobot$ can only paint $\fstile$ if each $\fsrobot_1,\ldots,\fsrobot_n$
    moves in such a way that tile $\fstile'$ above or below $\fstile$ becomes clear.
    Using the fact that a rectangular portion inside a rectangular grid
    has to be painted, it follows that the set of tiles that must not be painted
    is pairwise connected.
    Therefore, we can move each robot $\fsrobot_i$ from its current tile
    $\fstile_i'$ to $\fstile_i$ in such a way that after moving each robot,
    tile $\fstile$ becomes clear.
    The tuple $(\fprobotat{\fsrobot}{\fstile'}, \fppainted{\fstile}{\fscolor})$
    implies $G_r(s)$ in $s\in S_1^3$ in the admissible chain that consists of moving each robot $\fsrobot_i$
    from $\fstile_i'$ to $\fstile_i$ in such a way
    that moving all of them clears tile $\fstile'$, followed by
    moving $\fsrobot$ to $\fstile'$, and painting $\fstile$, i.e.,
    \begin{align*}
        (&(\fprobotat{\fsrobot}{\fstile}, \fprobothas{\fsrobot}{\fscolor'}), \\
         &(\fprobotat{\fsrobot}{\fstile}, \fprobothas{\fsrobot}{\fscolor}), \\
         &(\fprobotat{\fsrobot_1}{\fstile_1}, \fprobothas{\fsrobot}{\fscolor}),\ldots, \\
         &(\fprobotat{\fsrobot_n}{\fstile_n}, \fprobothas{\fsrobot}{\fscolor}), \\
         &(\fprobotat{\fsrobot}{\fstile'}, \fprobothas{\fsrobot}{\fscolor}), \\
         &(\fprobotat{\fsrobot}{\fstile'}, \fppainted{\fstile}{\fscolor}))
    \end{align*}
    We obtain sketch width $2$ because all tuples
    in admissible chains have size of at most $2$.
\end{proof}

\subsection{TPP}

In the Traveling Purchaser Problem (TPP) domain, there is a set of places
that can either be markets or depots, a set of trucks, and a set of goods
\cite{gerevini-et-al-aij2009}. The places are connected via a roads, allowing
trucks to drive between them.
If a truck $\tstruck$ is at place $\tsplace$,
then atom $\tpat{\tstruck}{\tsplace}$ holds.
Each market $\tsplace$ sells specific quantities of goods, e.g.,
atom $\tponsale{\tsgood}{\tsplace}{2}$ represents
that market $\tsplace$ sells two quantities of good $\tsgood$.
If there is a truck $\tstruck$ available at market $\tsplace$,
it can buy a fraction of the available quantity of good $\tsgood$,
making $\tsgood$ available to be loaded into $\tstruck$,
while the quantity available at $\tsplace$ decreases accordingly,
i.e., atom $\tponsale{\tsgood}{\tsplace}{1}$ and
$\tpreadytoload{\tsgood}{\tsplace}{1}$ hold afterwards.
The trucks can unload the goods at any depot,
effectively increasing the number of stored goods, e.g.,
atom $\tpstored{\tsgood}{1}$ becomes false,
and $\tpstored{\tsgood}{2}$ becomes true,
indicating that two quantities of good $\tsgood$ are stored.
The goal is to store specific quantities of specific goods.

\siw fails in TPP because loading sufficiently many quantities of a
single good can require buying and loading them from different markets.
Making the goods available optimally requires taking the direct route to each market
followed by buying the quantity of goods. Thus,
the problem width is bounded by the number of quantities needed.

Consider the set of features
$\Phi = \{\tfloaded, \tfstored\}$ where
$\tfloaded$ is the number of goods not stored in any truck
of which some quantity remains to be stored, and
$\tfstored$ is the sum of quantities of goods that remain to be stored.
The sketch rules in $R_\Phi$ are defined as:
\begin{align*}
    r_1 &= \prule{\GT{\tfloaded}}{\DEC{\tfloaded}} \\
    r_2 &= \prule{\GT{\tfstored}}{\UNK{\tfloaded},\DEC{\tfstored}}
\end{align*}
Rule $r_1$ says that loading any quantity of a good
that remains to be stored is good.
Rule $r_2$ says that storing any quantity of a good
that remains to be stored is good.

\begin{theorem}
    The sketch for the TPP domain is well-formed and has width $1$.
\end{theorem}
\begin{proof}
    The features are goal separating because
    $\EQ{\tfstored}$ holds in state $s$ iff $s$ is a goal state.
    We show that the sketch $\sketch$ is terminating by iteratively eliminating rules:
    $r_2$ decreases the numerical feature $\tfstored$ which no other rule increments,
    so we eliminate $r_2$ and mark $\tfstored$.
    Now only $r_1$ remains and we can eliminate it since it decreases $\tfloaded$,
    which is never incremented.

    It remains to show that the sketch width is $1$.
    Consider any TPP instance $P$.
    In the initial situation $s$, the feature conditions of at least one rule $r$ are true
    and the corresponding subgoal $G_r(s)$ is nonempty.
    Furthermore, whenever we use $r_1$ in some state $s$ to get to the next subgoal $G_{r_1}(s)$,
    we know that in this subgoal the feature conditions of $r_2$ must be true and its subgoal is nonempty,
    and it can be the case that the feature conditions of $r_1$ remain true and its subgoal is nonempty.
    At some point, the subgoal of $r_2$ is the overall goal of the problem.
    Next, regardless of which rule $r$ defines the closest subgoal $G_r(s)$ for an $R$-reachable state $s$,
    we show that $P[s]$ with subgoal $G_r(s)$ in $R$-reachable state $s$ has width $1$.

    We first consider rule $r_1$.
    Intuitively, we show that loading a good that is not yet loaded
    but of which some quantity remains to be stored in a depot has width at most $1$.
    Consider states $S_1\subseteq S$ where the feature conditions of $r_1$ are true,
    i.e., states where there is no truck that has a good $\tsgood$ loaded
    but of which some quantity remains to be stored in a depot.
    With $G_{r_1}(s)$ we denote the subgoal states of
    $r_1$ in $s\in S_1$, i.e., states where some quantity $q_l$ of
    $\tsgood$ is loaded into a truck $\tstruck$.
    The tuple $\tploaded{\tsgood}{\tstruck}{q_l}$ implies $G_{r_1}(s)$ in $s\in S_1$
    in the admissible chain that consists of moving $\tstruck$ from its
    current place $\tsplace_1$ to the closest market $\tsplace_n$ that has
    $\tsgood$ available, ordered descendingly by their distance to
    $\tsplace_n$, buying $q_b$ quantities of $\tsgood$, loading $q_l$
    quantities of $\tsgood$, i.e., $(\tpat{\tstruck}{\tsplace_1},
    \allowbreak \ldots, \allowbreak \tpat{\tstruck}{\tsplace_n},
    \allowbreak \tpreadytoload{\tsgood}{\tsplace_n}{q_b}, \allowbreak
    \tploaded{\tsgood}{\tstruck}{q_l})$. Note that loading $q_l$
    quantities can be achieved optimally by buying $q_b\geq q_l$ quantities.

    Next, we consider rule $r_2$.
    Intuitively, we show that storing a good of which some quantity remains to be stored in a depot has width at most $1$.
    Consider states $S_2\subseteq S$ where the feature conditions of $r_2$ are true
    and some quantity of a good is loaded that remains to be stored, i.e.,
    states where some quantity of a good $\tsgood$ remains to be stored in a depot,
    and some quantity $q_l$ of $\tsgood$ is loaded into a truck $\tstruck$ because it has width $1$ (see above).
    With $G_{r_2}(s)$ we denote the subgoal states of $r_2$ in $s\in S_2$, i.e.,
    states where the stored quantity $q_s$ of $\tsgood$ has increased from $q_s$ to $q_s'$
    using a fraction of the loaded quantity $q_l'\leq q_l$.
    The tuple $\tpstored{\tsgood}{q_s'}$ implies $G_{r_2}(s)$ in $s\in S_2$ in the admissible chain
    that consists of moving $\tstruck$ from its current place $\tsplace_1$
    to the closest depot at place $\tsplace_n$,
    ordered descendingly by their distance to $\tsplace_n$,
    storing $q_l'$ quantities of $\tsgood$, i.e.,
    $(\tpat{\tstruck}{\tsplace_1},\ldots,\tpat{\tstruck}{\tsplace_n}, \tpstored{\tsgood}{q_s'})$.

    We obtain sketch width $1$ because all tuples
    in admissible chains have a size of at most $1$.
\end{proof}

\subsection{Barman}

In the Barman domain \cite{linareslopez-et-al-aij2015},
there is a set of shakers, a set of shots, and a set of dispensers
where each dispenses a different ingredient.
There are recipes of cocktails each consisting of two ingredients, e.g.,
the recipe for cocktail $\bscocktail$ consists of ingredients $\bsingredient_1,\bsingredient_2$.
The goal is to serve beverages, i.e., ingredients and/or cocktails.
A beverage $\bsbeverage$ is served in shot $\bsshot$ if $\bsshot$ contains $\bsbeverage$.
An ingredient $\bsingredient$ can be filled into shot $\bsshot$
using one of the dispensers if $\bsshot$ is clean.
Producing a cocktail $\bscocktail$ with a shaker $\bsshaker$ requires
both ingredients $\bsingredient_1,\bsingredient_2$ of $\bscocktail$ to be in $\bsshaker$.
In such a situation, shaking $\bsshaker$ produces $\bscocktail$.
Pouring a cocktail from $\bsshaker$ into shot $\bsshot$ requires $\bsshot$ to be clean.
The barman robot has two hands which limits the number of shots and shakers it can hold at the same time.
Therefore, the barman often has to put down an object before it can grasp a different object.
For example, assume that the barman holds the shaker $\bsshaker$ and some shot $\bsshot'$
and assume that ingredient $\bsingredient$ must be filled into shot $\bsshot$.
Then the barman has to put down either $\bsshaker$ or $\bsshot'$ so that it can pick up $\bsshot$ with hand $\bshand$.
As in the Barman tasks from previous IPCs, we assume that there is only a single shaker
and that it is initially empty.

Consider the set of features $\Phi = \{\bfunserved, \bfused,
\bfconsistentone, \bfconsistenttwo \}$ where $\bfunserved$ is the number
of unserved beverages, $\bfused$ is the number of used shots, i.e.,
shots with a beverage different from the one mentioned in the goal,
$\bfconsistentone$ is true iff the first recipe ingredient of an unserved
cocktail is in the shaker, and $\bfconsistenttwo$ is true iff both recipe
ingredients of an unserved cocktail are in the shaker. We define the
following sketch rules for $R_\Phi$:
\begin{align*}
    r_1 &= \prule{\neg \bfconsistentone}{\UNK{\bfused}, \bfconsistentone}, \\
    r_2 &= \prule{\bfconsistentone,\neg \bfconsistenttwo}{\UNK{\bfused},\bfconsistenttwo}, \\
    r_3 &= \prule{\GT{\bfused}}{\DEC{\bfused}}, \\
    r_4 &= \prule{\GT{\bfunserved}}{\DEC{\bfunserved}, \UNK{\bfconsistentone}, \UNK{\bfconsistenttwo}}.
\end{align*}
Rule $r_1$ says that filling an ingredient into the shaker is good
if this ingredient is mentioned in the first part of the recipe of an unserved cocktail.
Rule $r_2$ says the same for the second ingredient, after the first ingredient has been added.
Requiring the ingredients to be filled into the shaker in a fixed order ensures
bounded width, even for arbitrary-sized recipes.
Rule $r_3$ says that cleaning shots is good and rule $r_4$ says that serving an ingredient or cocktail is good.

\begin{theorem}
    The sketch for the Barman domain is well-formed and has width $2$.
\end{theorem}

\subsection{Grid}
In the Grid domain \cite{mcdermott-aimag2000}, a single robot operates in a grid-structured world.
There are keys and locks distributed over the grid cells.
The robot can move to a cell $\gscell$ above, below, left or right of its current cell
if $\gscell$ does not contain a closed lock or another robot.
The robot can drop, pick or exchange keys at its current cell
and can only hold a single key $\gskey$ at any time.
Keys and locks have different shapes
and the robot, holding a matching key, can open a lock
when standing on a neighboring cell.
The goal is to move keys to specific target locations that can be locked initially.
Initially, it is possible reach every lock for the unlock operation.
\siw fails in this domain when goals need to be undone, i.e.,
a key has to be picked up from its target location
to open a lock that is necessary for picking or moving a different key.

Consider the set of features
$\Phi = \{\gflocked, \gfkey, \gfopen, \gftarget \}$ where
$\gflocked$ is the number of locked grid cells,
$\gfkey$ is the number of misplaced keys,
$\gfopen$ is true iff the robot holds a key for which there is a closed lock, and
$\gftarget$ is true iff the robot holds a key that must be placed at some target grid cell.
We define the sketch rules for $R_\Phi$ as:
\begin{align*}
    r_1 &= \prule{\GT{\gflocked}}{\DEC{\gflocked},\UNK{\gfkey},\UNK{\gfopen},\UNK{\gftarget}} \\
    r_2 &= \prule{\EQ{\gflocked},\GT{\gfkey}}{\DEC{\gfkey},\UNK{\gfopen},\UNK{\gftarget}} \\
    r_3 &= \prule{\GT{\gflocked}, \neg \gfopen}{\gfopen, \UNK{\gftarget}} \\
    r_4 &= \prule{\EQ{\gflocked},\neg \gftarget}{\UNK{\gfopen}, \gftarget}
\end{align*}
Rule $r_1$ says that unlocking grid cells is good.
Rule $r_2$ says that placing a key at its target cell
is good after opening all locks.
Rule $r_3$ says that picking up a key that can be used
to open a locked grid cell is good if there are locked grid cells.
Rule $r_4$ says that picking up a misplaced key
is good after opening all locks.

\begin{theorem}
    The sketch for the Grid domain is well-formed and has width $1$.
\end{theorem}

\subsection{Childsnack}

In the Childsnack domain \cite{vallati-et-al-ker2018}, there is a set of
contents, a set of trays, a set of gluten-free breads, a set of regular
breads that contain gluten, a set of gluten-allergic children, a set of
children without gluten allergy, and a set of tables where the children
sit. The goal is to serve the gluten-allergic children with sandwiches
made of gluten-free bread and the non-allergic children with either type
of sandwich.

The Childsnack domain has large bounded width because moving an empty tray
is possible at any given time. The goal serialization fails because it
gets trapped in deadends when serving non-allergic children with
gluten-free sandwiches while leaving insufficiently many gluten-free
sandwiches for the allergic children.

Consider the set of features
$\Phi = \{\cfallergicchild, \cfregularchild, \cfglutenfreesandwichkitchen, \cfsandwichkitchen, \cfglutenfreesandwichtray, \cfsandwichtray \}$
where $\cfallergicchild$ is the number of unserved gluten-allergic children, $\cfregularchild$
is the number of unserved non-allergic children,
$\cfglutenfreesandwichkitchen$ holds iff there is a gluten-free sandwich
in the kitchen, $\cfsandwichkitchen$ holds iff there is a regular
sandwich in the kitchen, $\cfglutenfreesandwichtray$ holds iff there is
a gluten-free sandwich on a tray, and $\cfsandwichtray$ holds iff there
is any sandwich on a tray.
We define the following sketch rules $R_\Phi$:
\begin{align*}
    r_1 &= \prule{\GT{\cfallergicchild},\neg \cfglutenfreesandwichkitchen,\neg \cfglutenfreesandwichtray}{\cfglutenfreesandwichkitchen, \cfsandwichkitchen} \\
    r_2 &= \prule{\EQ{\cfallergicchild}, \GT{\cfregularchild}, \neg \cfsandwichkitchen, \neg \cfsandwichtray}{\cfsandwichkitchen} \\
    r_3 &= \prule{\GT{\cfallergicchild},\cfglutenfreesandwichkitchen, \neg \cfglutenfreesandwichtray}{\UNK{\cfglutenfreesandwichkitchen}, \UNK{\cfsandwichkitchen}, \cfglutenfreesandwichtray, \cfsandwichtray} \\
    r_4 &= \prule{\EQ{\cfallergicchild}, \GT{\cfregularchild}, \cfsandwichkitchen, \neg \cfsandwichtray}{\UNK{\cfglutenfreesandwichkitchen}, \UNK{\cfsandwichkitchen}, \UNK{\cfglutenfreesandwichtray}, \cfsandwichtray} \\
    r_5 &= \prule{\GT{\cfallergicchild}, \cfglutenfreesandwichtray}{\DEC{\cfallergicchild},\UNK{\cfglutenfreesandwichtray}, \UNK{\cfsandwichtray}} \\
    r_6 &= \prule{\EQ{\cfallergicchild}, \GT{\cfregularchild}, \cfsandwichtray}{\DEC{\cfregularchild}, \UNK{\cfglutenfreesandwichtray}, \UNK{\cfsandwichtray}}
\end{align*}
Rule $r_1$ says that making a gluten-free sandwich is good
if there is an unserved gluten-allergic child
and if there is no other gluten-free sandwich currently being served.
Rule $r_2$ says the same thing for non-allergic children
after all gluten-allergic children have been served
and the sandwich to be made is not required to be gluten free.
Rules $r_3$ and $r_4$ say that putting a gluten-free (resp.\ regular)
sandwich from the kitchen onto a tray is good if there is none on a tray yet.
Rule $r_5$ says that serving gluten-allergic children before non-allergic children is good
if there is a gluten-free sandwich available on a tray.
Rule $r_6$ says that serving non-allergic children afterwards is good.

\begin{theorem}
    The sketch for the Childsnack domain is well-formed and has width $1$.
\end{theorem}

\subsection{Driverlog}

In the Driverlog domain \cite{long-fox-jair2003}, there is a set of drivers, trucks, packages, road
locations and path locations. The two types of locations form two strongly
connected graphs and the two sets of vertices overlap. The road graph is
only traversable by trucks, while the path graph is only traversable by
drivers. A package can be delivered by loading it into a truck, driving
the truck to the target location of the package followed by unloading the
package. Driving the truck requires a driver to be in the truck. Not only
packages, but also trucks and drivers can have goal locations.
\siw fails because it can be necessary to undo previously achieved goals,
like moving a truck away from its destination to transport a package. The
following sketch induces a goal ordering such that an increasing subset of
goal atoms never needs to be undone.

Consider the set of features
$\Phi = \{\dfpackages, \dftrucks, \dfdrivergoal, \dfdrivertruck, \dfboarded, \dfloaded \}$ where
$\dfpackages$ is the number of misplaced packages,
$\dftrucks$ is the number of misplaced trucks,
$\dfdrivergoal$ is the sum of all distances of drivers to their respective goal locations,
$\dfdrivertruck$ is the minimum distance of any driver to a misplaced truck,
$\dfboarded$ is true iff there is a driver inside of a truck, and
$\dfloaded$ is true iff there is a misplaced package in a truck.
We define the sketch rules $R_\Phi$ as follows:
\begin{align*}
    r_1 &= \prule{\GT{\dfpackages}, \neg\dfboarded}{\UNK{\dfdrivergoal}, \UNK{\dfdrivertruck}, \dfboarded} \\
    r_2 &= \prule{\GT{\dfpackages}, \neg\dfloaded}{\UNK{\dftrucks}, \UNK{\dfdrivergoal}, \UNK{\dfdrivertruck}, \dfloaded} \\
    r_3 &= \prule{\GT{\dfpackages}}{\DEC{\dfpackages}, \UNK{\dftrucks}, \UNK{\dfdrivergoal}, \UNK{\dfdrivertruck}, \UNK{\dfloaded}} \\
    r_4 &= \prule{\EQ{\dfpackages}, \GT{\dftrucks}, \GT{\dfdrivertruck}}{\UNK{\dfdrivergoal}, \DEC{\dfdrivertruck}, \UNK{\dfboarded}} \\
    r_5 &= \prule{\EQ{\dfpackages},\GT{\dftrucks}, \EQ{\dfdrivertruck}}{\DEC{\dftrucks}, \UNK{\dfdrivergoal}, \UNK{\dfdrivertruck}} \\
    r_6 &= \prule{\EQ{\dfpackages},\EQ{\dftrucks},\GT{\dfdrivergoal}}{\DEC{\dfdrivergoal}, \UNK{\dfboarded}}
\end{align*}
Rule $r_1$ says that letting a driver board any truck is good
if there are undelivered packages and there is no driver boarded yet.
Rule $r_2$ says that loading an undelivered package is good.
Rule $r_3$ says that delivering a package is good.
Rule $r_4$ says that moving any driver closer to being
in a misplaced truck is good after having delivered all packages.
Rule $r_5$ says that driving a misplaced truck to its target location
is good once all packages are delivered.
Rule $r_6$ says that moving a misplaced driver closer
to its target location is good
after having delivered all packages and trucks.

\begin{theorem}
    The sketch for the Driverlog domain is well-formed and has width $1$.
\end{theorem}

\subsection{Schedule}

In the Schedule domain \cite{bacchus-aimag2001}, there is a set of objects that
can have different values for the following attributes: shape, color,
surface condition, and temperature. Also, there is a set of machines where
each is capable of changing an attribute with the side effect that other
attributes change as well. For example, rolling an object changes its
shape to cylindrical and has the side effect that the color changes to
uncolored, any surface condition is removed, and the object becomes hot.
Often, there are multiple different work steps for achieving a specific
attribute of an object. For example, both rolling and lathing change an
object's shape to cylindrical. But rolling makes the object hot, while
lathing keeps its temperature cold. Some work steps are only possible if
the object is cold. Multiple work steps can be scheduled to available
machines, which sets the machine's status to occupied. All machines become
available again after a single do-time-step action. The goal is to change
the attributes of objects.

\siw fails in Schedule because it gets trapped into deadends when an object's temperature
becomes hot, possibly blocking other required attribute changes.
The following sketch uses this observation and defines
an ordering over achieved attributes where
first, the desired shapes are achieved,
second, the desired surface conditions are achieved, and
third, the desired colors are achieved.

Consider the set of features $\Phi = \{\sfshape, \sfsurface, \sfcolor, \sfhot, \sfoccupied\}$ where
$\sfshape$ is the number of objects with wrong shape,
$\sfsurface$ is the number of objects with wrong surface condition,
$\sfcolor$ is the number of objects with wrong color,
$\sfhot$ is the number of hot objects, and
$\sfoccupied$ is true iff there is an object scheduled or a machine occupied.
We define the following sketch rules $R_\Phi$:
\begin{align*}
    r_1 &= \prule{\GT{\sfshape}}{\DEC{\sfshape},\UNK{\sfsurface},\UNK{\sfcolor}, \sfoccupied} \\
    r_2 &= \prule{\EQ{\sfshape},\GT{\sfsurface}}{\DEC{\sfsurface},\UNK{\sfcolor},\sfoccupied} \\
    r_3 &= \prule{\EQ{\sfshape},\EQ{\sfsurface},\GT{\sfcolor}}{\DEC{\sfcolor}, \sfoccupied} \\
    r_4 &= \prule{\sfoccupied}{\neg \sfoccupied}
\end{align*}
Rule $r_1$ says that achieving an object's goal shape is good.
Rule $r_2$ says that achieving an object's goal surface condition is good
after achieving all goal shapes.
Rule $r_3$ says that achieving an object's goal color is good
after achieving all goal shapes and surface conditions.
Rule $r_4$ says that making objects and machines available is good.
Note that $r_4$ does not decrease the sketch width but it decreases the
search time by decreasing the search depth. Note also that $\sfhot$ never
occurs in any rule because we want its value to remain constant.

\begin{theorem}
    The sketch for the Schedule domain is well-formed and has width $2$.
\end{theorem}

\section{Experiments}

Even though the focus of our work is on proving polynomial runtime bounds
for planning domains theoretically, we evaluate in this section how these
runtime guarantees translate into practice. We implemented \siwR in the
LAPKT planning system \cite{ramirez-et-al-misc2015} and use the Lab toolkit
\cite{seipp-et-al-zenodo2017} for running experiments on Intel Xeon Gold
6130 CPU cores. For each planning domain, we use the tasks from previous
IPCs. For each planner run, we limit time and memory by 30 minutes and 4
GiB. The benchmark set consists of a subset of tractable classical planning
domains from the satisficing track of the 1998-2018 IPC where top goal serialization using SIW fails.

The main question we want to answer empirically is how much an SIW
search benefits from using policy sketches. To this end, we compare
\siw{}(2) and \siwR{}(2) with the sketches for the planning domains
introduced above. We use a width bound of $k$=$2$, since \siw{}($k$) and
\siwR{}($k$) are too slow to compute in practice for larger values of $k$.
We also include two state-of-the-art planners, \lama \cite{richter-westphal-jair2010} and \bfws
\cite{lipovetzky-et-al-aaai2017}, to show that the planning tasks in our benchmark set
are hard to solve with the strongest planners.

\newcommand{\numtasks}[1]{\small{(#1)}}
\newcommand{\asep}{\hskip 15pt}
\setlength{\tabcolsep}{4pt}
\begin{table*}[tbp]
    \setlength{\cmidrulekern}{15pt}
    \centering
    \begin{tabular}{@{}l@{\asep}rrrr@{\asep}rrrr@{\asep}rr@{\asep}rr@{}}
        & \multicolumn{4}{c}{\siw{}(2)} & \multicolumn{4}{c}{\siwR{}(2)} & \multicolumn{2}{c}{\lama{} ~~~~~} & \multicolumn{2}{c}{\bfws{}} \\
        \cmidrule(r){2-5}
        \cmidrule(r){6-9}
        \cmidrule(r){10-11}
        \cmidrule(){12-13}
        Domain & S & T & AW & MW & S & T & AW & MW & S & T & S & T \\
        \midrule
        Barman \numtasks{40} & 0 & -- & -- & -- & 40 & 0.9 & 1.17 & 2 & 40 & 505.3 & 40 & 162.8 \\
        Childsnack \numtasks{20} & 0 & -- & -- & -- & 20 & 10.8 & 1.00 & 1 & 6 & 2.6 & 8 & 216.9 \\
        Driverlog \numtasks{20} & 8 & 0.5 & 1.68 & 2 & 20 & 0.8 & 1.00 & 1 & 20 & 7.6 & 20 & 4.2 \\
        Floortile \numtasks{20} & 0 & -- & -- & -- & 20 & 0.2 & 1.25 & 2 & 2 & 9.9 & 2 & 176.3 \\
        Grid \numtasks{5} & 1 & 0.1 & 2.00 & 2 & 5 & 0.1 & 1.00 & 1 & 5 & 3.6 & 5 & 3.7 \\
        Schedule \numtasks{150} & 62 & 1349.1 & 1.10 & 2 & 150 & 54.7 & 1.17 & 2 & 150 & 15.3 & 150 & 151.4 \\
        TPP \numtasks{30} & 11 & 74.7 & 2.00 & 2 & 30 & 0.4 & 1.00 & 1 & 30 & 16.5 & 29 & 99.6 \\
    \end{tabular}
    \caption{Comparison of \siw{}(2), \siwR{}(2), the first iteration of
    \lama{}, and \bfws{}. It shows the number of solved tasks (S), the
    maximum runtime in seconds for a successful run (T), the average
    effective width over all encountered subtasks (AW), and the maximum
    effective width over all encountered subtasks (MW).}
    \label{table:results}
\end{table*}

Table~\ref{table:results} shows results for the four planners. We see that
the maximum effective width (MW) for \siwR{}(2) never exceeds the
theoretical upper bounds we established in the previous section. For
the domains with sketch width $2$, the average effective width (AW) is
always closer to $1$ than to $2$.

In the comparison we must keep in mind
that \siwR is not a domain-independent planner
as it uses a suitable sketch for each domain.
\siw{}(2) solves none of the instances in three domains (Barman,
Childsnack, Floortile) because the problem width is too large. In
the other four domains, it never solves more than half of the tasks. In
contrast, \siwR{}(2) solves all tasks and is usually very fast. For
example, \siw{}(2) needs 74.7 seconds to solve the eleventh hardest TPP
task, while \siwR{}(2) solves all 30 tasks in at most 0.4 seconds. This
shows that our sketch rules capture useful information and that the sketch
features are indeed cheap to compute.

Even with the caveats about planner comparisons in mind, the results from
Table~\ref{table:results} show that policy sketches usually let \siwR
solve the tasks from our benchmark set much faster than state-of-the-art
domain-independent planners. The only exception is Schedule, where \lama
has a lower maximum runtime than \siwR. The main bottleneck for \siwR in
Schedule is generating successor states. Computing feature valuations on
the other hand takes negligible time.

Overall, our results show that adding domain-specific knowledge in the
form of sketches to a width-based planner allows it to solve whole problem
domains very efficiently. This raises interesting research questions about
whether we can learn sketches automatically to transform \siwR into a
domain-independent planner that can reuse previously acquired information.

\section{Related Work}

We showed that a bit of knowledge about the subgoal structure of a domain,
expressed elegantly in the form of compact sketches,
can go a long way in solving the instances of a domain
efficiently, in provable polynomial time.
There are other approaches for expressing domain control knowledge
for planning in the literature, and we review them here.

\newcommand{\tlplan}{\ensuremath{\textsc{TLplan}}}

The distinction between the actions that are ``good'' or ``bad''
in a fixed tractable domain can often be characterized explicitly.
Indeed, the so-called \textbf{general policies}, unlike sketches,  provide such
a classification of all possible state transitions $(s,s')$ over the problems in $\Q$ \cite{frances-et-al-aaai2021},
and ensure that the goals can always  be reached by following any good transitions.
Sketch rules have the same syntax as policy rules,
but they do not constraint state transitions but define subgoals.

Logical approaches to domain control have been used to provide partial information
about good and bad state transitions in terms of suitable formulas \cite{bacchus-kabanza-aij2000,kvarnstrom-doherty-amai2000}.
For example, for the Schedule domain, one may have a formula in \textbf{linear temporal logic} (LTL)
expressing that objects that need to be lathed and painted should
not be painted in the next time step, since lathing removes the paint.
This partial information about good and bad transitions
can then be used by a forward-state search planner to heavily
prune the state space. A key difference between these formulas
and sketches is that sketch rules are not about state transitions
but about subgoals, and hence they structure the search for plans
in a different way, in certain cases ensuring a polynomial search.

\intextcite{baier-et-al-aaai2008} combine control knowledge and preference
formulas to improve search behavior and obtain plans of high quality,
according to user preferences. The control knowledge is given in the Golog
language and defines subgoals such that a planner has to fill in the
missing parts. Since the control knowledge is compiled directly to PDDL,
they are able to leverage off-the-shelve planners. The user preferences
are encoded in an LTL-like language. Like our policy sketches, their
approach can be applied to any domain. However, policy sketches aim at
ensuring polynomial searches in tractable domains.

Hierarchical task networks or \textbf{HTNs} are used mainly for expressing
general top-down strategies for solving classes of planning problems
\cite{erol-et-al-aaai1994,nau-et-al-jair2003,georgievski-aiello-aij2015}. The domain knowledge
is normally expressed in terms of a hierarchy of methods
that have to be decomposed into primitive methods that cannot
be decomposed any further. While the solution strategy expressed in HTNs does not have to be complete,
it is often close to complete, as otherwise the search for decompositions easily becomes intractable.
For this reason, crafting good and effective HTNs encodings is not easy.
For example, the HTN formulation of the Barman domain
in the 2020 Hierarchical Planning Competition \cite{holler-et-al-icapswhp2019} contains
10 high-level tasks (like \emph{AchieveContainsShakerIngredient}), 11 primitive tasks (like
\emph{clean-shot}) and 22 methods (like \emph{MakeAndPourCocktail}).
In contrast, the PDDL version of Barman has only 12 action schemas,
and the sketch above has 4 rules over 4 linear features.
Note, however, that comparing different forms of control knowledge in terms of their compactness
is not well-defined.

Finally, the need to represent the common subgoal structure of dynamic domains
arises also in reinforcement learning (RL), where knowledge gained in the solution of some
domain instances can be applied to speed up the learning of solutions to new
instances of the same family of tasks \cite{finn-et-al-icml2017}.
In recent work in deep RL (DRL)  these representations, in the form of general \textbf{intrinsic
reward functions} \cite{singh-et-al-ieeetamd2021},  are expected to be learned from suitable DRL architectures
\cite{zheng-et-al-icml2020}. Sketches provide a convenient high-level alternative to describe
common subgoal structures, but opposed to the related work in DRL, the policy sketches above
are not learned but are written by hand. We leave the challenge of automatically learning sketches
as future work and describe it briefly below.

\section{Conclusions and Future Work}

We have  shown that the language of policy sketches as introduced by Bonet and Geffner
provides a simple, elegant, and powerful way for expressing the common subgoal structure of
many planning domains. The \siwR algorithm can then solve these domains effectively,
in provable polynomial time, where \siw fails either because the problems
are not easily serializable in terms of the top goals or because
some of the resulting  subproblems have a high width. A big  advantage
of pure width-based algorithms like \siw and \siwR is that unlike
other planning-based methods they can be used to plan with simulators
in which the structure of states  is available but the structure of  actions
is not.\footnote{
A minor difference then is that the novelty tests in IW($k$)
are not exponential in $k-1$ but in $k$.}

A logical next step in this line of work is to \textbf{learn sketches automatically}.
In principle, methods similar to those used for learning general policies can be applied.
These methods rely on using the state language (primitive PDDL
predicates) for defining a large pool of Boolean and numerical features
via a description logic grammar \cite{baader-et-al-2003}, from which
the features $\Phi$ are selected and over which the general policies $\pi_\Phi$
are constructed \cite{frances-et-al-aaai2021}. We have actually analyzed the features
used in the sketches given above and have noticed that they  can all  be
obtained from a feature pool generated in this way. A longer-term challenge
is to learn the sketches automatically when using the same inputs as DRL algorithms,
where there is no state representation language. Recent works that learn
first-order symbolic languages from black box states or from states represented by images
\cite{bonet-geffner-ecai2020,asai-icaps2019} are important first steps in that direction.

\section*{Acknowledgments}

This work was partially supported by an ERC Advanced Grant (grant agreement no.\ 885107),
by project TAILOR, funded by EU Horizon 2020 (grant agreement no.\ 952215), and
by the Knut and Alice Wallenberg Foundation.
Hector Geffner is a Wallenberg Guest Professor at Link{\"o}ping University, Sweden.
We used compute resources from the Swedish National Infrastructure for
Computing (SNIC), partially funded by the Swedish Research Council through
grant agreement no.\ 2018-05973.

\section{Appendix}

In this section, we provide proofs for the theorems
that claim the sketches are well-formed and have bounded width,
and describe how the features are computed.

\appendix

\section{Proofs for Sketches}

\begin{subappendices}
\subsection{Grid}
\begin{theorem}
    The sketch for the Grid domain is well-formed and has sketch width $1$.
\end{theorem}
\begin{proof}
    The features $\Phi$ are goal separating because the feature valuation
    $\EQ{\gfkey}$ holds in state $s$ iff $s$ is a goal state.
    We show that the sketch is terminating by iteratively eliminating rules:
    $r_1$ decreases $\gflocked$ which no other rule increases, so we eliminate $r_1$ and
    mark $\gflocked$. Now $r_2$ can be eliminated because it decreases
    $\gfkey$ which no remaining rule increases. We can now eliminate $r_3
    = C \mapsto E$ because it changes the Boolean feature $\gfopen$ and
    the only other remaining rule $r_4 = C'\mapsto E'$ may restore the
    value of $\gfopen$, but this can only happen finitely often, since
    $\gflocked$ is marked and $\GT{\gflocked} \in C$ and $\EQ{\gflocked}
    \in C'$. Now only $r_4$ remains and we can eliminate it since it
    changes $\gftarget$, which is never changed back.

    It remains to show that the sketch width is $1$.
    Consider any Grid instance $P$ with states $S$.
    Note that depending on the initial state $s$ the feature conditions of at least one rule $r$ are true
    in $s$ and its subgoal $G_r(s)$ are nonempty.
    We first consider rule $r_3$.
    Intuitively, we show that picking up a key that can be used to open some closed lock has width $1$.
    Consider states $S_1\subseteq S$ where the feature conditions of $r_3$ are true, i.e.,
    states where there is a closed lock and the robot does not hold a key $\gskey$
    that can be used to open a closed lock.
    With $G_{r_3}(s)$ we denote the subgoal states of $r_3$ in $s\in S_1$,
    i.e., states where the robot holds $\gskey$.
    The tuple $\gpholding{\gskey}$ implies $G_{r_3}(s)$ in $s\in S_1$
    in the admissible chain that consists of changing the position of the robot
    from the current position $\gscell_1$ to the position $\gscell_n$ of $\gskey$
    ordered by the distance to $\gscell_n$, and followed by exchanging or picking $\gskey$, i.e.,
    $(\gpatrobot{\gscell_1},\ldots,\gpatrobot{\gscell_n},\gpholding{\gskey})$.
    Note that the feature conditions of $r_1$ are true in states $s'$
    in subgoal $G_{r_3}(s)$ with nonempty subgoal states $G_{r_1}(s')$
    because the number of closed locks remains greater than $0$.

    Next, we consider rule $r_1$.
    Intuitively, we show that opening a closed lock has width $1$.
    Consider states $S_2\subseteq S$ where feature conditions of $r_1$ are true and the robot holds
    a key $\gskey$ that can be used to open a closed lock $\gslock$.
    We can transform states where the robot holds no key into a state from $S_2$ by
    letting it pick a key with width $1$ (see above).
    With $G_{r_1}(s)$ we denote the subgoal states of $r_1$ in $s\in S_2$, i.e., states where $\gslock$ is open.
    The tuple $\gpopen{\gslock}$ implies $G_{r_1}(s)$ in $s\in S_2$
    in the admissible chain that consists of changing the position of the robot
    from its current position $\gscell_1$ to a position $\gscell_n$ next to lock $\gslock$
    ordered by the distance to $\gscell_n$, i.e.,
    $(\gpatrobot{\gscell_1},\ldots,\gpatrobot{\gscell_n},\gpopen{\gslock})$.
    Note that either there are still closed locks that can be opened by repeated
    usage of rules $r_1$ and $r_3$ or the feature conditions of $r_2$ or $r_4$ are true
    in states $s'$ in the subgoal $G_{r_1}(s)$ with nonempty subgoal states $G_{r_2}(s')$
    or $G_{r_4}(s')$ respectively because there are misplaced keys.
    Hence, it remains to show that if all locks are open then well-placing keys has width $1$.

    Next, we consider rule $r_4$.
    Intuitively, we show that picking up a key that is not at its target cell has width $1$.
    Consider states $S_3\subseteq S$ where the feature conditions of $r_4$ are true, i.e.,
    states where all locks are open and the robot does not hold a misplaced key.
    With $G_{r_4}(s)$ we denote the subgoal states of $r_4$ in $s\in S_3$, i.e.,
    states where the robot holds $\gskey$.
    The tuple $\gpholding{\gskey}$ implies $G_{r_4}(s)$ in $s\in S_3$
    in the admissible chain that consists of changing the position of the robot
    from the current position $\gscell_1$ to the position $\gscell_n$ of $\gskey$
    ordered by the distance to $\gscell_n$, and followed by exchanging or picking $\gskey$, i.e.,
    $(\gpatrobot{\gscell_1},\ldots,\gpatrobot{\gscell_n},\gpholding{\gskey})$.
    Note that the feature conditions of $r_2$ are true in states $s'$
    in subgoal $G_{r_4}(s)$ with nonempty subgoal states $G_{r_2}(s')$
    because the number of misplaced keys remains greater than $0$.

    Finally, we consider rule $r_2$.
    Intuitively, we show that moving a key to its target cell has width $1$.
    Consider states $S_4\subseteq S$ where the feature conditions of $r_2$ are true
    and the robot holds a misplaced key $\gskey$.
    As before, we can transform states $s' \notin S_4$ into such a state $s$ by picking up $\gskey$ with width $1$.
    With $G_{r_2}(s)$ we denote the subgoal states of $r_2$ in $s\in S_4$, i.e.,
    states where $\gskey$ is at its target cell.
    The tuple $\gpat{\gskey}{\gscell_n}$ implies $G_{r_2}(s)$ in $s\in S_4$
    in the admissible chain that consists of changing the position of the robot
    from its current position $\gscell_1$ to the key's target cell $\gscell_n$
    ordered by the distance to $\gscell_n$, followed by exchanging or dropping the key at $\gscell_n$, i.e.,
    $(\gpatrobot{\gscell_1},\ldots,\gpatrobot{\gscell_n},\gpat{\gskey}{\gscell_n})$.

    We obtain sketch width $1$ because all tuples in admissible chains have size $1$.
\end{proof}

\subsection{Barman}

\begin{theorem}
    The sketch for the Barman domain is well-formed and has sketch width $2$.
\end{theorem}
\begin{proof}
    The features $\Phi$ are goal separating because
    $\EQ{\bfunserved}$ holds in state $s$ iff $s$ is a goal state.
    We show that the sketch is terminating by iteratively eliminating rules:
    first, we eliminate $r_4$ because it
    decreases the numerical feature $\bfunserved$ that no rule increases.
    Next, rules $r_1$ and $r_2$ can be eliminated because both
    change a Boolean feature that no remaining rule changes in the opposite direction.
    Last, we eliminate the $r_3$ because it decrements the numerical feature $\bfused$.

    It remains to show that the sketch width is $2$.
    Consider any Barman instance $P$ with states $S$.
    In the initial state $s$ the feature condition of $r_4$ are true,
    and the subgoal $G_{r_4}(s)$ is nonempty.
    Note that using $r_4$ to reach a subgoal decreases the number of unserved beverages
    until the overall goal is reached. Hence, $r_4$ can be seen as the goal counter.
    If the beverage to be served is a cocktail or if the shots are dirty,
    then this subproblem can be further decomposed into smaller subproblems using rules $r_1,r_2,r_3$ as follows.
    Producing a cocktail requires filling the shaker with correct ingredients
    and can be achieved by successively reaching the subgoals defined by rules $r_1$ and $r_2$.
    Next, if the shot required for serving was made dirty during this process,
    then $r_3$ defines the subgoal of cleaning it again.
    Finally, $r_4$ defines the subgoal of serving the cocktail.

    We first consider rule $r_3$.
    Intuitively, we show that shots are cleaned with width at most $1$.
    Consider all states $S_1\subseteq S$ where the feature conditions of $r_3$ are true, i.e.,
    states where there is a used shot $\bsshot$
    such that $\bpused{\bsshot}{\bsbeverage}$ holds for some beverage $\bsbeverage$
    that is not supposed to be in $\bsshot$ according to the goal description.
    With $G_{r_3}(s)$ we denote the subgoal states of $r_3$ in $s\in S_1$, i.e.,
    states where $\bsshot$ is clean.
    We do a case distinction over states $S_1$.
    First, consider states $S_1^1\subseteq S_1$
    where the barman is holding $\bsshot$ in hand $\bshand$.
    The tuple $\bpclean{\bsshot}$ implies $G_{r_3}(s)$ for all $s\in S_1^1$
    in the admissible chain that consists of cleaning $\bsshot$, i.e.,
    $(\bpholding{\bshand}{\bsshot}, \bpclean{\bsshot})$.
    Second, consider states $S_1^2\subseteq S_1$
    where the barman must grasp $\bsshot$ with empty hand $\bshand$ first.
    The same tuple $\bpclean{\bsshot}$ implies $G_{r_3}(s)$ for all $s\in S_1^2$
    in the admissible chain that consists of picking $\bsshot$, and cleaning $\bsshot$, i.e.,
    $(\bpontable{\bsshot}, \bpholding{\bshand}{\bsshot}, \bpclean{\bsshot})$.
    Last, consider states $S_1^3\subseteq S_1$
    where the barman must exchange $\bsshot'$ with $\bsshot$ in hand $\bshand$ first.
    The same tuple $\bpclean{\bsshot}$ implies $G_{r_3}(s)$ for all $s\in S_1^3$
    in the admissible chain that consists of putting down $\bsshot'$, picking up $\bsshot$, cleaning $\bsshot$, i.e.,
    $(\bpholding{\bshand}{\bsshot'}, \bpontable{\bsshot'}, \bpholding{\bshand}{\bsshot}, \bpclean{\bsshot})$.
    It also follows that we can reduce the set of $R$-reachable states in our analysis
    to those where the container is already grasped if only a single container is affected.

    Next, we consider rule $r_1$.
    Intuitively, we show that filling the first ingredient into the shaker
    for producing a required cocktail has width at most $2$.
    Consider states $S_2\subseteq S$ where the feature conditions of $r_1$ are true and required shots are clean, i.e.,
    states where no ingredient $\bsingredient_1$ consistent with the first part of some unserved cocktail $\bscocktail$'s recipe is in the shaker $\bsshaker$. We do not need to consider states where
    required shots are not clean because a shot can be cleaned with width $1$ (see above).
    With $G_{r_1}(s)$ we denote the subgoal states of $r_1$ in $s$, i.e.,
    states where an ingredient $\bsingredient_1$ consistent with the first recipe part of some unserved cocktail $\bscocktail$ is inside $\bsshaker$.
    The tuple $(\bpcontains{\bsshaker}{\bsingredient_1}, \bpshakerlevel{\bsshaker}{\bclevelone})$
    implies $G_{r_1}(s)$ in the admissible chain that consists of cleaning $\bsshaker$,
    putting down $\bsshaker$, picking a clean shot $\bsshot$,
    filling $\bsingredient_1$ into $\bsshot$ using the corresponding dispenser,
    and pouring $\bsshot$ into $\bsshaker$, i.e.,
    \begin{align*}
        (&(\bpholding{\bshand}{\bsshaker}, \bpshakerlevel{\bsshaker}{\bcleveltwo}), \\
        &(\bpholding{\bshand}{\bsshaker}, \bpempty{\bsshaker}), (\bpholding{\bshand}{\bsshaker}, \bpclean{\bsshaker}), \\
        &(\bpontable{\bsshaker}, \bpclean{\bsshaker}), (\bpholding{\bshand}{\bsshot}, \bpclean{\bsshaker}), \\
        &(\bpcontains{\bsshot}{\bsingredient_1}, \bpused{\bsshot}{\bsingredient_1}), \\
        &(\bpcontains{\bsshaker}{\bsingredient_1}, \bpshakerlevel{\bsshaker}{\bclevelone}) ).
    \end{align*}
    Note that the feature conditions of rule $r_2$ are true in states $s'$
    in subgoal $G_{r_1}(s)$ with nonempty subgoal $G_{r_2}(s')$.

    Next, we consider rule $r_2$.
    Intuitively, we show that filling the second ingredient into the shaker
    for producing a required cocktail has width at most $1$.
    Consider states $S_3\subseteq S$ where the feature conditions of $r_2$ are true and required shots are clean, i.e.,
    states where the first ingredient consistent
    with the recipe of an unserved cocktail $\bscocktail$ is in the shaker $\bsshaker$,
    and required shots are clean because a shot can be cleaned with width $1$ (see above).
    With $G_{r_2}(s)$ we denote the subgoal states of $r_2$ in $s$, i.e.,
    states where an ingredient $\bsingredient_2$ is inside $\bsshaker$
    such that both ingredients in $\bsshaker$ are consistent
    with the recipe of an unserved cocktail $\bscocktail$.
    The tuple $(\bpcontains{\bsshaker}{\bsingredient_2}, \bpshakerlevel{\bsshaker}{\bcleveltwo})$
    implies $G_{r_2}(s)$ in the admissible chain that consists of putting down $\bsshaker$,
    grasping $\bsshot$, filling $\bsingredient_2$ into $\bsshot$ using the corresponding dispenser,
    and pouring $\bsshot$ into $\bsshaker$, i.e.,
    \begin{align*}
        (&(\bpholding{\bshand}{\bsshaker}, \bpshakerlevel{\bsshaker}{\bclevelone}), \\
        &(\bpontable{\bsshaker}, \bpshakerlevel{\bsshaker}{\bclevelone}), \\
        &(\bpholding{\bshand}{\bsshot}, \bpclean{\bsshot}), \\
        &(\bpcontains{\bsshot}{\bsingredient_2}, \bpused{\bsshot}{\bsingredient_2}), \\
        &(\bpcontains{\bsshaker}{\bsingredient_2}, \bpshakerlevel{\bsshaker}{\bcleveltwo}) ).
    \end{align*}
    Finally, we consider rule $r_4$, where we show intuitively that serving a beverage has width at most $1$.
    We do a case distinction over all states $S_4$ where the feature conditions of $r_4$ are true, i.e.,
    states where there is an unserved ingredient or an unserved cocktail.
    First, consider states $S_4^1\subseteq S_4$
    where there is an unserved ingredient $\bsingredient$.
    $G_{r_4}^1(s)$ is the set of subgoal states for $r_4$ in $s\in S_4^1$
    where $\bsingredient$ is served.
    The tuple $\bpcontains{\bsshot}{\bsingredient}$ implies $G_{r_4}^1(s)$
    in the admissible chain that consists of filling $\bsingredient$
    into $\bsshot$ using the corresponding dispenser, i.e.,
    $(\bpclean{\bsshot}, \bpcontains{\bsshot}{\bsingredient})$.
    Last, consider states $S_4^2\subseteq S_4$
    where there is an unserved cocktail $\bscocktail$,
    respective ingredients are in the shaker using the results of rule $r_1,r_2$, and
    required shots are clean using the results of rule $r_3$.
    With $G_{r_4}^2(s)$ we denote the subgoal states of $r_4$ in $s\in S_4^2$
    where $\bscocktail$ is served.
    The tuple $\bpcontains{\bsshot}{\bscocktail}$ implies $G_{r_4}^2(s)$
    in the admissible chain that consists of putting down $\bsshot$ (or any other shot)
    because shaking requires only the shaker $\bsshaker$ to be held,
    shaking $\bsshaker$, and pouring $\bsshaker$ into $\bsshot$, i.e.,
    \begin{align*}
    (\bpholding{\bshand}{\bsshot}, \bpontable{\bsshot}, \\
    \bpcontains{\bsshaker}{\bscocktail}, \bpcontains{\bsshot}{\bscocktail})
    \end{align*}
    We obtain sketch width $2$ because all tuples
    in admissible chains have a size of at most $2$.
\end{proof}

\subsection{Childsnack}
\begin{theorem}
    The sketch for the Childsnack domain is well-formed and has sketch width $1$.
\end{theorem}
\begin{proof}
    The features are goal separating because the feature valuations
    $\EQ{\cfallergicchild}$ and $\EQ{\cfregularchild}$ hold in state $s$
    iff $s$ is a goal state. We show that the sketch is terminating by
    iteratively eliminating rules: $r_5$ decreases the numerical feature
    $\cfallergicchild$ which no other rule increments, so we eliminate
    $r_5$ and mark $\cfallergicchild$. Similarly, $r_6$ decreases the
    numerical feature $\cfregularchild$ which no other rule increments, so
    we eliminate $r_6$ and mark $\cfregularchild$. Then rules $r_4$
    changes $\cfsandwichtray$ and no remaining rules changes
    $\cfsandwichtray$ in the opposite direction, so we eliminate $r_4$.
    Likewise, we eliminate $r_3$ because it changes
    $\cfglutenfreesandwichtray$, which no remaining rule can change back.
    Last, we eliminate rules $r_1$ and $r_2$ because they change
    $\cfglutenfreesandwichkitchen$ resp.\ $\cfsandwichkitchen$, and no
    remaining rule can change the values in the opposite direction.

    It remains to show that the sketch width is $1$.
    Consider any Childsnack instance with states $S$.
    Note that if there is an unserved gluten-allergic child in the initial state
    then rules $r_1,r_3,r_5$ define subgoals for serving a gluten-allergic child.
    If there is no unserved gluten-allergic child but there is an unserved non-allergic child
    then rules $r_2,r_4,r_6$ define subgoals for serving a non-allergic child.
    In the following, we first show that serving a gluten-free sandwich to
    a gluten-allergic child has width $1$
    and deduce the case of serving a non-allergic child from it.

    We first consider rule $r_1$. Intuitively, we show that producing a
    gluten-free sandwich has width $1$. Consider states $S_3\subseteq S$
    where the feature conditions of $r_1$ are true, i.e., states where there is an unserved
    gluten-allergic child $\cschild$ and there is no gluten-free sandwich
    available in $\cckitchen$ nor on a tray. With $G_{r_1}(s)$ we denote
    the subgoal states of $r_1$ in $s\in S_3$, i.e., states where
    gluten-free sandwich $\cssandwich$ is available in $\cckitchen$. The
    tuple $\cpnoglutensandwich{\cssandwich}$ implies $G_{r_1}(s)$ in $s\in
    S_3$ in the admissible chain that consists of producing $\cssandwich$,
    i.e., $(\cpnotexists{\cssandwich)},
    \cpnoglutensandwich{\cssandwich})$.
    Note that the feature conditions of $r_3$ are true in states $s'$
    in the subgoal $G_{r_1}(s)$ with nonempty subgoal $G_{r_3}(s')$.

    Next, we consider rule $r_3$. Intuitively, we show that moving a
    gluten-free sandwich from the kitchen onto a tray has width $1$.
    Consider states $S_2\subseteq S$ where the feature conditions of $r_3$ are true, i.e.,
    states where there is an unserved gluten-allergic child $\cschild$ and
    there is a gluten-free sandwich $\cssandwich$ available in
    $\cckitchen$. With $G_{r_3}(s)$ we denote the subgoal states of $r_3$
    in $s\in S_2$, i.e., states where $\cssandwich$ is on $\cstray$. The
    tuple $\cpontray{\cssandwich}{\cstray}$ implies $G_{r_3}(s)$ in $s\in
    S_2$ in the admissible chain that consists of moving $\cstray$ from
    $\cstable$ to $\cckitchen$, putting $\cssandwich$ onto $\cstray$, i.e.,
    $(\cpat{\cstray}{\cstable}, \cpat{\cstray}{\cckitchen},
    \cpontray{\cssandwich}{\cstray})$.
    Note that the feature conditions of $r_5$ are true in states $s'$
    in the subgoal $G_{r_3}(s)$ with nonempty subgoal $G_{r_5}(s')$.

    Next, we consider rule $r_5$. Intuitively, we show that serving a
    gluten-allergic child if there is a gluten-free sandwich is available
    on a tray has width $1$. Consider states $S_1\subseteq S$ where the feature conditions of $r_5$
    are true, i.e., states where there is an unserved gluten-allergic
    child $\cschild$ and there is a gluten-free sandwich $\cssandwich$
    on a tray $\cstray$. With $G_{r_5}(s)$ we denote the subgoal
    states of $r_5$ in $s\in S_1$, i.e., states where $\cschild$ is
    served. The tuple $\cpserved{\cschild}$ implies $G_{r_5}(s)$ in $s\in
    S_1$ in the admissible chain that consists of moving $\cstray$ from
    $\cckitchen$ to $\cstable$, serving $\cschild$ with $\cssandwich$,
    i.e., $(\cpontray{\cssandwich}{\cstray}, \cpat{\cstray}{\cstable},
    \cpserved{\cschild})$.
    Note that if all gluten allergic children are served in this way
    by using rules $r_1,r_3,r_5$ then either $G$ was reached
    or there are unserved non-allergic children.
    In the latter case, the problem is very similar to
    the one we considered above and rules $r_2,r_4,r_6$ define the
    corresponding subgoals to serve a non-allergic child. We omit the
    details but provide the admissible chains that are necessary to
    conclude the proof:
    the tuple $\cpserved{\cschild}$ implies $G_{r_6}(s)$ in the admissible chain
    $(\cpontray{\cssandwich}{\cstray}, \cpat{\cstray}{\cstable}, \cpserved{\cschild})$.
    The tuple $\cpontray{\cssandwich}{\cstray}$ implies $G_{r_4}(s)$ in the admissible chain
    $(\cpat{\cstray}{\cstable}, \cpat{\cstray}{\cckitchen}, \cpontray{\cssandwich}{\cstray})$.
    The tuple $\cpatkitchensandwich{\cssandwich}$ implies $G_{r_2}(s)$ in the admissible chain
    $(\cpnotexists{\cssandwich)}, \cpatkitchensandwich{\cssandwich})$.

    As a result, we get sketch width $1$ because all tuples
    in admissible chains have size of at most $1$.
\end{proof}

\subsection{Driverlog}
\begin{theorem}
    The sketch for the Driverlog domain is well-formed and has sketch width $1$.
\end{theorem}
\begin{proof}
    The features are goal separating because the feature valuations
    $\EQ{\dfpackages},\EQ{\dftrucks},\EQ{\dfdrivergoal}$ hold in state $s$ iff $s$ is a goal state.
    We show that the sketch is terminating by iteratively eliminating rules:
    $r_3$ decreases the numerical feature $\dfpackages$ that no other remaining rule increments,
    so we eliminate $r_3$ and mark $\dfpackages$.
    We can now eliminate $r_5 = C\mapsto E$
    because it decreases the numerical feature $\dftrucks$
    and the only other remaining rule $r_2 = C'\mapsto E'$ arbitrarily changes $\dftrucks$,
    but this can only happen finitely many times, since
    $\dfpackages$ is marked and $\EQ{\dfpackages}\in C$ and $\GT{\dfpackages}\in C'$.
    Next, we can eliminate $r_2$
    because it sets the Boolean feature $\dfloaded$
    and no other remaining rule changes $\dfloaded$ in the opposite direction.
    We can now eliminate $r_4 = C\mapsto E$
    because it decreases the numerical feature $\dfdrivertruck$
    and the only other remaining rule $r_1 = C'\mapsto E'$ arbitrarily changes $\dfdrivertruck$,
    but this can only happen finitely many times, since
    $\dfpackages$ is marked and $\EQ{\dfpackages}\in C$ and $\GT{\dfpackages}\in C'$.
    Next, we can now eliminate $r_6 = C\mapsto E$
    because it decreases the numerical feature $\dfdrivergoal$
    and the only other remaining rule $r_1 = C'\mapsto E'$ arbitrarily changes $\dfdrivergoal$,
    but this can only happen finitely many times, since
    $\dfpackages$ is marked and $\EQ{\dfpackages}\in C$ and $\GT{\dfpackages}\in C'$.
    Last, we eliminate the remaining rule $r_1$ because it sets the Boolean feature $\dfboarded$ to true.

    It remains to show that the sketch width is $1$.
    Consider any Driverlog instance with states $S$.
    If there are misplaced packages in the initial state,
    then rule $r_3$ decrements the number of misplaced packages.
    Therefore, we show that moving packages to their
    target location has width $1$.
    Consider states $S_1\subseteq S$ where there is a misplaced package $\dspackage$
    at location $\dslocation_m$ with target location $\dslocation_o$.
    We do a three-way case distinction over all states $S_1$
    and show that moving a package to its target location has width $1$.
    First, consider rule $r_1$. Intuitively, we show that boarding some driver
    into a truck has width $1$. Consider states $S_1^1\subseteq S_1$
    where the feature conditions of rule $r_1$ are true, i.e., states there is no driver boarded into any truck.
    With $G_{r_1}(s)$ we denote the subgoal states of $r_1$ in $s\in S_1^1$, i.e.,
    states where a driver $\dsdriver$ is boarded into a truck $\dstruck$.
    The tuple $\dpdriving{\dsdriver}{\dstruck}$ implies $G_{r_1}(s)$ in $s\in S_1^1$
    in admissible chain that consists of moving $\dsdriver$ from $\dslocation_1$ to $\dslocation_n$,
    each step decreasing the distance to $\dslocation_n$, boarding $\dsdriver$ into $\dstruck$, i.e.,
    \begin{align*}
        (\dpat{\dsdriver}{\dslocation_1},\ldots,\dpat{\dsdriver}{\dslocation_n},\dpdriving{\dsdriver}{\dstruck}).
    \end{align*}
    Second, consider rule $r_2$. Intuitively, we show that loading a misplaced package
    into a truck has width $1$. Consider states $S_1^2\subseteq S_1$
    where the feature conditions of rule $r_2$ are true and where $\dsdriver$ is boarded into truck $\dstruck$ at location $l_n$, i.e.,
    no misplaced package is loaded, and $\dsdriver$ is boarded into $\dstruck$ at location $l_n$
    because boarding has $\dsdriver$ into $\dstruck$ if there is a misplaced package has width $1$ (see above).
    With $G_{r_2}(s)$ we denote the subgoal states of $r_2$ in $s\in S_1^2$, i.e.,
    states where $\dspackage$ is loaded into $\dstruck$.
    The tuple $\dpin{\dspackage}{\dstruck}$ implies $G_{r_2}(s)$ in $s\in S_1^2$ in the admissible chain
    that consists of driving $\dstruck$ from $\dslocation_n$
    to $\dslocation_m$, each step decreasing the distance to $\dslocation_m$,
    loading $\dspackage$ into $\dstruck$, i.e.,
    \begin{align*}
        (\dpat{\dstruck}{\dslocation_n},\ldots,\dpat{\dstruck}{\dslocation_m},\dpin{\dspackage}{\dstruck}).
    \end{align*}
    Third, consider rule $r_3$. Intuitively, we show that moving a package to
    it target location has width $1$. Consider states $S_1^3\subseteq S_1$
    where the feature conditions of rule $r_3$ are true and where $\dspackage$ and $\dsdriver$ is in $\dstruck$ at $\dslocation_m$, i.e.,
    states where $\dspackage$ and $\dsdriver$ is in $\dstruck$ at $\dslocation_m$
    because loading driver and misplaced package has width $1$ (see above).
    With $G_{r_3}(s)$ we denote the subgoal states of $r_3$ in $s\in S_1^3$, i.e.,
    states where $\dspackage$ is at location $\dslocation_o$.
    The tuple $\dpat{\dspackage}{\dslocation_o}$ implies $G_{r_3}(s)$ in the admissible chain
    that consists of driving $\dstruck$ from $\dslocation_m$
    to $\dslocation_o$, each step decreasing the distance to $\dslocation_o$,
    and unloading $\dspackage$, i.e.,
    \begin{align*}
        (\dpat{\dstruck}{\dslocation_m},\ldots,\dpat{\dstruck}{\dslocation_o},\dpat{\dspackage}{\dslocation_o}).
    \end{align*}

    Now, consider states $S_2$ where all packages are at their respective target location
    and there is a misplaced truck $\dstruck$ at location $l_n$ with target location $l_m$.
    This can either be the case in the initial state
    or after moving the packages because it requires to use trucks.
    We do a two-way case distinction over all states $S_2$
    and show that moving a truck to its target location has width $1$.
    Consider rule $r_4$. Intuitively, we show that boarding a driver into
    a misplaced truck without using any truck has width $1$.
    Consider states $S_2^1\subseteq S_2$ where the feature conditions of rule $r_4$ are true, i.e.,
    where there is a driver $\dsdriver$ at location $\dslocation_1$
    with nonzero distance until being boarded into $\dstruck$.
    With $G_{r_4}(s)$ we denote the subgoal states of $r_4$ in $s\in S_2^1$, i.e.,
    states where $\dsdriver$ is one step closer to being boarded into $\dstruck$.
    There are three possible admissible chains that must be considered.
    (1) unboarding $\dsdriver$ from some truck $\dstruck'$,
    i.e., tuple $\dpat{\dsdriver}{\dslocation_1}$ implies $G_{r_4}(s)$ in $s\in S_2^1$ in the admissible chain
    $(\dpdriving{\dsdriver}{\dstruck'}, \dpat{\dsdriver}{\dslocation_1})$,
    (2) moving $\dsdriver$ closer to $\dslocation_n$
    over $\dslocation_{i-1}$ to $\dslocation_i$,
    i.e., tuple $\dpat{\dsdriver}{\dslocation_i}$ implies $G_{r_4}(s)$ in $s\in S_2^1$ in the admissible chain
    $(\dpat{\dsdriver}{\dslocation_{i-1}}, \dpat{\dsdriver}{\dslocation_i})$, and
    (3) boarding $\dsdriver$ into $\dstruck$ at $\dslocation_n$,
    i.e., $\dpdriving{\dsdriver}{\dstruck}$ implies $G_{r_4}(s)$ in $s\in S_2^1$ in the admissible chain
    $(\dpat{\dsdriver}{\dslocation_n}, \dpdriving{\dsdriver}{\dstruck})$.
    Second, consider rule $r_5$. Intuitively, we show that moving a misplaced truck
    to its target location has width $1$.
    Consider states $S_2^2\subseteq S_2$ where the feature conditions of rule $r_5$ are true
    and where some driver is boarded into $\dstruck$, i.e.,
    states where $\dsdriver$ is boarded intro $\dstruck$ at $\dslocation_n$.
    With $G_{r_5}(s)$ we denote the subgoal states of $r_5$ in $s\in S_2^2$, i.e.,
    states where $\dstruck$ is at its target location.
    The tuple $\dpat{\dstruck}{\dslocation_n}$ implies $G_{r_5}(s)$ in $s\in S_2^2$
    in the admissible chain that consists of moving $\dstruck$ from $\dslocation_n$
    to $\dslocation_m$, each step decreasing the distance to $\dslocation_m$, i.e.,
    $(\dpat{\dstruck}{\dslocation_n},\ldots,\dpat{\dstruck}{\dslocation_m})$.

    Now, consider states $S_3$ where all packages and trucks are at their respective target location
    and there is a misplaced driver $\dsdriver$ boarded or unboarded
    at location $l_1$ with target location $l_n$.
    This can either be the case in the initial state
    or after moving the packages and trucks.
    Consider rule $r_6$. Intuitively, we show that moving a driver to its target location
    without using any truck has width $1$.
    With $G_{r_6}(s)$ we denote the subgoal states of $r_6$ in $s\in S_3$, i.e.,
    states where $\dsdriver$ is at its target location.
    There are two possible admissible chains that must be considered.
    (1) unboarding $\dsdriver$ at location $\dslocation_1$,
    i.e., tuple $\dpat{\dsdriver}{\dslocation_1})$ implies $G_{r_6}(s)$ in $s\in S_3$ in the admissible chain
    $(\dpdriving{\dsdriver}{\dstruck}, \dpat{\dsdriver}{\dslocation_1})$, and
    (2) moving $\dsdriver$ closer from location $\dslocation_{i-1}$ to $\dslocation_i$,
    i.e., tuple $\dpat{\dsdriver}{\dslocation_i}$ implies $G_{r_6}(s)$ in $s\in S_3$ in the admissible chain
    $(\dpat{\dsdriver}{\dslocation_{i-1}}, \dpat{\dsdriver}{\dslocation_i})$.

    We obtain sketch width $1$ because all tuples
    in admissible chains have size $1$.
    Note that when dropping rules $r_1$ and $r_2$, as well as features
    $\dfloaded$ and $\dfboarded$, the sketch width becomes $2$ because we
    must merge the three admissible chains of the first subproblem. When
    merging, tuples of size $2$ must be considered, each consisting of a
    location and whether the driver drives the truck or whether the
    package is loaded.
\end{proof}

\subsection{Schedule}
\begin{theorem}
    The sketch for the Schedule domain is well-formed and has sketch width $2$.
\end{theorem}
\begin{proof}
    The features are goal separating because the feature valuations
    $\EQ{\sfshape},\EQ{\sfsurface},\EQ{\sfcolor}$ hold in state $s$ iff $s$ is a goal state.
    We show that the sketch is terminating by iteratively eliminating rules:
    $r_1$ decreases the numerical feature $\sfshape$ that no other remaining rule increments,
    so we eliminate $r_1$ and mark $\sfshape$.
    $r_2$ decreases the numerical feature $\sfsurface$ that no other remaining rule increments,
    so we eliminate $r_2$ and mark $\sfsurface$.
    $r_3$ decreases the numerical feature $\sfcolor$ that no other remaining rule increments,
    so we eliminate $r_3$ and mark $\sfcolor$.
    Last, we eliminate the only remaining rule $r_4$
    because it sets the Boolean feature $\sfoccupied$ to false.

    It remains to show that the sketch width is $2$.
    Consider any Schedule instance with states $S$.
    First, consider states $S_1\subseteq S$ where the feature conditions of $r_4$ are true, i.e.,
    there is either a scheduled object or a machine occupied.
    This can be the case in the initial state or
    if an object is scheduled to be processed by a machine.
    With $G_{r_4}(s)$ we denote the subgoal states of $r_4$ in $s\in S_1$, i.e.,
    states where no object is scheduled and no machine is occupied while not removing any other achieved goal atom.
    The tuple $\spnotscheduled{\ssobject}$ implies $G_{r_4}(s)$ in $s\in S_1$ in the admissible chain
    that consists of the action that performs a single time step, i.e.,
    $(\spscheduled{\ssobject}, \spnotscheduled{\ssobject})$.

    Now, consider states $S_2\subseteq S$ where the feature conditions of $r_1$ are true
    and there is no object scheduled and no occupied machine, i.e.,
    states where there is an object $\ssobject$
    that has shape $x$ that is not the shape $y$ mentioned in the goal,
    and there is no object scheduled and no occupied machine
    because this can be achieved with width $1$ (see above).
    With $G_{r_1}(s)$ we denote the set of subgoal states of $r_1$ in $s\in S_2$, i.e.,
    states where $\ssobject$ has shape $y$ while not removing any other achieved goal atom.
    The tuple $(\spshape{\ssobject}{z},\sptemperature{x}{\sstemperature})$
    implies $G_{r_1}(s)$ in $s\in S_2$ in the admissible chain
    that consists of changing the shape with procedures that do not affect the temperature of $\ssobject$, i.e.,
    \begin{align*}
        (&(\spshape{\ssobject}{y},\sptemperature{x}{\sstemperature}), \\
        &(\spshape{\ssobject}{z},\sptemperature{x}{\sstemperature}))
    \end{align*}

    Now, consider states $S_3\subseteq S$ where the feature conditions of $r_2$ are true,
    there is no object scheduled and no machine occupied, and
    objects have their correct shape, i.e.,
    states where there is an object $\ssobject$
    that has surface $x$ that is not the surface $y$ mentioned in the goal,
    there is no object scheduled and no occupied machine
    because this can be achieved with width $1$ (see above), and
    all objects have their correct shape because changing the shape has width $2$ (see above).
    With $G_{r_2}(s)$ we denote the set of subgoal states of $r_2$ in $s\in S_3$, i.e.,
    states where $\ssobject$ has surface $y$ while not removing any other achieved goal atom.
    The tuple $(\spsurfacecondition{\ssobject}{z},\sptemperature{x}{\sstemperature})$
    implies $G_{r_2}(s)$ in $s\in S_3$ in the admissible chain
    that consists of changing the surface with procedures that do not affect the temperature of $\ssobject$, i.e.,
    \begin{align*}
        (&(\spsurfacecondition{\ssobject}{y},\sptemperature{x}{\sstemperature}), \\
        &(\spsurfacecondition{\ssobject}{z},\sptemperature{x}{\sstemperature})).
    \end{align*}

    Now, consider states $S_4\subseteq S$ where the feature conditions of $r_3$ are true,
    there is no object scheduled and no machine occupied,
    objects have their correct shape, and
    objects have their correct surface, i.e.,
    states where there is an object $\ssobject$
    that has color $x$ that is not the color $y$ mentioned in the goal,
    there is no object scheduled and no occupied machine because this can be achieved with width $1$ (see above),
    all objects have their correct shape because changing the shape has width $2$ (see above), and
    all objects have their desired surface because changing the surface has width $2$ (see above).
    With $G_{r_3}(s)$ we denote the set of subgoal states of $r_3$ in $s\in S_3$, i.e.,
    states where $\ssobject$ has color $y$ while not removing any other achieved goal atom.
    The tuple $(\sppainted{\ssobject}{z},\sptemperature{x}{\sstemperature})$
    implies $G_{r_2}(s)$ in $s\in S_2$ in the admissible chain
    that consists of changing the surface with procedures that do not affect the temperature of $\ssobject$, i.e.,
    \begin{align*}
            (&(\sppainted{\ssobject}{y},\sptemperature{x}{\sstemperature}), \\
            &(\sppainted{\ssobject}{z},\sptemperature{x}{\sstemperature}))
    \end{align*}
    Note that $r_3$ achieves the goal when the color of the last object
    changes to the color mentioned in the goal.

    We obtain sketch width $2$ because all tuples
    in admissible chains have a size of at most $2$.
\end{proof}
\end{subappendices}

\section{Feature Definitions and Grammar}

Following \intextcite{frances-et-al-aaai2021} the definition of the features
used in the sketches is given in terms of a grammar based on
predicates of each planning domain and description logics \cite{baader-et-al-2003}.
Description logics consider two types of object classes: concepts and roles.
Concepts correspond to unary predicates, whereas roles correspond to binary predicates.
The following definition of syntax and semantics is based on those given by \intextcite{frances-et-al-aaai2021}.

\begin{subappendices}
\subsection{Syntax and Semantics}

We make use of a richer description logic than \intextcite{frances-et-al-aaai2021},
and unlike them, there is no attempt at learning the features automatically
but writing down their definitions by hand.
The grammar extension is the introduction of primitive concepts and roles
from projections of primitive predicates of the planning domain,
which may have arity greater than $2$.
Consider concepts $C,D$ and roles $R,S$ and the universe $\Delta$
containing all objects in the planning instance.
The set of possible concepts and roles for each state $s$ are inductively defined as:

\begin{itemize}
    \item A \emph{primitive concept} $p[i]$ is a concept that denotes the set of
    objects occuring at index $i$ in ground atoms of predicate $p$ in state $s$.
    \item A \emph{primitive role} $p[i,j]$ is a role that denotes the set of pairs of
    objects occuring at index $i$ and $j$ in ground atoms of predicate $p$ in state $s$.
    \item The \emph{universal concept} $\top$ and the bottom concept $\bot$ are concepts
    with denotations $\top = \Delta$, $\bot = \emptyset$.
    \item The \emph{union} $C\sqcup D$, \emph{intersection} $C\sqcap D$, \emph{negation} $\neg C$
    are concepts with denotations $(C\sqcup D) = C\cup D$, $(C\sqcap D) = C\cap D$, $(\neg C) = \Delta\setminus C$.
    \item The \emph{existential abstraction} $\exists R.C$ and the \emph{universal abstraction} $\forall R.C$
    are concepts with denotations $(\exists R.C) = \{a\mid\exists b:(a,b)\in R\land b\in C\}$,
    $\forall R.C = \{a\mid\forall b:(a,b)\in R\rightarrow b\in C\}$.
    \item If $a$ is a constant in the planning domain,
    the \emph{nominal} $a$ is a concept that represents $\{a\}$.
    \item The \emph{union} $R\sqcup S$, \emph{intersection} $R\sqcap D$, \emph{complement} $\neg R$
    are roles with denotations $(R\sqcup S) = R\cup S$, $(R\sqcap D) = R\cap S$, $\neg R = (\Delta\times\Delta)\setminus R$.
    \item The \emph{role-value map} $R = S, R\subseteq S$ are concepts with denotations
    $(R = S) = \{a\mid\forall b:(a,b)\in R\leftrightarrow (a,b)\in S \}$, and
    $(R\subseteq S) = \{a\mid\forall b:(a,b)\in R\rightarrow (a,b)\in S\}$.
    \item The \emph{composition} $R\circ S$ is a role with denotation
    $R\circ S = \{(a,c)\mid (a,b)\in R\land (b,c)\in S \}$.
    \item The \emph{inverse} $R^{-1}$ is a role with denotation
    $R^{-1} = \{(b, a)\mid (a,b)\in R \}$.
    \item The \emph{transitive closure} $R^+$, \emph{reflexive-transitive closure} are roles with denotations
    $R^+ = \bigcup_{n\geq 1}(R)^n$, and $R^* = \bigcup_{n\geq 0}(R)^n$ where the iterated composition
    is defined as $(R)^0 = \{(d,d)\mid d\in\triangle \}$ and $(R)^{n+1} = (R)^n\circ R$.
    \item The \emph{restriction} $R\vert_{C}$ is a role with denotation
    $R\vert_{C} = R\sqcap (\Delta\times C)$.
    \item The \emph{identity} $\mathit{id}(C)$ is a role with denotation
    $\mathit{id}(C) = \{(a,a)\mid a\in C\}$.
    \item The \emph{difference} $C\setminus D$, $R\setminus S$ is a concept and a role respectively with denotation
    $C\setminus D = C\sqcap \neg B$, $R\setminus S = R\sqcap \neg S$.\footnote{does not increase expressiveness but makes it more convenient to express the difference between concepts and roles.}
    \item The \emph{extraction} $R[i]$ with $i\in\{0,1\}$ is a concept with denotation
    $R[0] = \exists R.\top$ and $R[1] = \exists R^{-1}.\top$.\footnote{does not increase expressiveness but makes it more convenient to capture concepts occuring at a specific position in a role.}
\end{itemize}
Furthermore, for each concept $C$ and role $R$ we allow for corresponding goal versions denoted by $C_g$ and $R_g$
that are evaluated in the goal of the planning instance instead of the state $s$, as described by \intextcite{frances-et-al-aaai2021}.

\subsection{From Concepts and Roles to Features}

We define Boolean and numerical features with an additional level of composition as follows.
Consider concepts $C,D$, roles $R,S,T$, and $X$ being either a role or a concept.

\begin{itemize}
    \item \emptyfeature{X} is true iff $|X|=0$.
    \item \countfeature{X} counts the number of elements in $X$.
    \item \conceptdistancefeature{C}{R}{D} is the smallest $n\in\mathbb{N}_0$
    s.t.\ there are objects $x_0,\ldots,x_n$,
    $x_0\in C$, $x_n\in D$ and $R(x_i, x_{i+1})$ for $i = 0,\ldots,n-1$.
    If no such $n$ exists then the feature evaluates to $\infty$.
    \item $\roledistancefeature{R}{S}{T}$ is the smallest $n\in\mathbb{N}_0$
    s.t.\ there are objects $x_0,\ldots,x_n$,
    there exists some $(a,x_0)\in R$, $(a,x_n)\in T$,
    and $S(x_i, x_{i+1})$ for $i = 0,\ldots,n-1$.
    If no such $n$ exists, the feature evaluates to $\infty$.
    \item $\sumroledistancefeature{R}{S}{T}:=\sum_{r\in R}\roledistancefeature{\{r\}}{S}{T}$,
    where the sum evaluates to $\infty$ if any term is $\infty$.
\end{itemize}

\subsection{Floortile}
Consider the following concepts and roles:
\begin{align*}
    x_1&\equiv (\mathit{painted}_g[0,1]\setminus\mathit{painted}[0,1])[0] \\
    x_2&\equiv (\mathit{left}[0]\sqcup \mathit{left}[1])\setminus \mathit{painted}_g[0])\\
    x_3&\equiv \mathit{up}[0,1]\sqcup\mathit{down}[0,1]\sqcup \mathit{id}(\mathit{left}[0]\sqcup \mathit{left}[1])\\
    x_4&\equiv ((x_3\vert_{x_1})^{-1}\vert_{x_1})^{-1} \\
    x_5&\equiv ((x_3\vert_{x_2})^{-1}\vert_{x_1})^{-1}
\end{align*}
Concept $x_1$ is the set of all unpainted tiles.
Concept $x_2$ is the set of all normal tiles that must not be painted.
Role $x_3$ is the set of pairs of tiles $(t, t')$ where $t$ is above or below of $t'$ and the identity $t = t'$.
Role $x_4$ is the set of pairs of tiles $(t, t')$ where $t$ is above or below of $t'$ and both are unpainted.
Role $x_5$ is the set of pairs of tiles $(t, t')$ where $t$ is unpainted and above or below of normal tile $t'$.
The features $\Phi = \{\ffinvariant, \ffunpaintedtiles \}$ used in the sketch for Floortile are constructed as follows:
\begin{align*}
    \ffinvariant&= \emptyfeature{x_1\setminus ((x_4)^*\circ x_5)[0]} \\
    \ffunpaintedtiles&= \countfeature{x_1}
\end{align*}

\subsection{TPP}

Consider the following concepts and roles:
\begin{align*}
    x_1&\equiv (\mathit{stored}_g[0,1]\setminus\mathit{stored}[0,1])[0] \\
    x_2&\equiv \mathit{next}[1]
\end{align*}
Concept $x_1$ is the set of goods of which some quantity remains to be stored.
Concept $x_2$ is the set of nonempty levels.
The features $\Phi = \{\tfloaded, \tfstored\}$ used in the sketch for TPP are constructed as follows:
\begin{align*}
    \tfloaded&\equiv \countfeature{x_1\setminus \exists\mathit{loaded}[0,2].x_2} \\
    \tfstored&\equiv \sumroledistancefeature{\mathit{stored}[0,1]}{\mathit{next}[1,0]}{\mathit{stored}_g[0,1]}
\end{align*}

\subsection{Barman}

Consider the following concepts and roles:
\begin{align*}
    x_1&\equiv (\mathit{contains}_g[0,1]\setminus\mathit{contains}[0,1]) \\
    x_2&\equiv \forall\mathit{cocktail}\text{-}\mathit{part1}[0,1].\exists\mathit{contains}[1,0].\mathit{shaker}\text{-}\mathit{level}[0] \\
    x_3&\equiv \forall\mathit{cocktail}\text{-}\mathit{part2}[0,1].\exists\mathit{contains}[1,0].\mathit{shaker}\text{-}\mathit{level}[0]
\end{align*}
Role $x_1$ is the set of beverages remain to be served paired with the corresponding shot that must be used.
Concept $x_2$ is the set of cocktails where the first ingredient mentioned in its recipe is in the shaker.
Concept $x_3$ is the set of cocktails where the second ingredient mentioned in its recipe is in the shaker.
The features $\Phi = \{\bfunserved, \bfused, \bfconsistentone, \bfconsistenttwo \}$ used in the sketch for Barman are constructed as follows:
\begin{align*}
    \bfunserved&\equiv \countfeature{x_1} \\
    \bfused&\equiv \countfeature{\mathit{used}[0]\setminus x_1[0]} \\
    \bfconsistentone&\equiv \neg\emptyfeature{x_2\sqcap x_1[1]} \\
    \bfconsistenttwo&\equiv \neg\emptyfeature{x_2\sqcap x_3\sqcap x_1[1]}
\end{align*}


\subsection{Grid}

Consider the following concepts and roles:
\begin{align*}
    x_1&\equiv \mathit{at}_g[0,1]\setminus\mathit{at}[0,1] \\
    x_2&\equiv \exists\mathit{key}\text{-}\mathit{shape}[0,1].\exists\mathit{lock}\text{-}\mathit{shape}[0,1].\mathit{locked}[0]
\end{align*}
Role $x_1$ is the set of misplaced key paired with their respective target location.
Concept $x_2$ is the set of keys for which a closed lock with the same shape as the key exists.
The features $\Phi = \{\gflocked, \gfkey, \gfopen, \gftarget \}$  used in the sketch for Grid are constructed as follows:
\begin{align*}
    \gflocked&\equiv \countfeature{\mathit{locked}[0]} \\
    \gfkey&\equiv \countfeature{x_1} \\
    \gfopen&\equiv \neg\emptyfeature{\mathit{holding}\sqcap x_2} \\
    \gftarget&\equiv \neg\emptyfeature{\mathit{holding}\sqcap x_1[0]}
\end{align*}

\subsection{Childsnack}

Consider the following concepts and roles:
\begin{align*}
    x_1 &\equiv \mathit{served}_g[0]\setminus\mathit{served}[0] \\
    x_2 &\equiv \mathit{no}\text{-}\mathit{gluten}\text{-}\mathit{sandwich}[0]
\end{align*}
Concept $x_1$ is the set of unserved children.
Concept $x_2$ is the set of gluten-free sandwiches.
The features $\Phi = \{\cfallergicchild, \cfregularchild, \cfglutenfreesandwichkitchen, \cfsandwichkitchen, \cfglutenfreesandwichtray, \cfsandwichtray \}$
used in the sketch for Childsnack are constructed as follows:
\begin{align*}
    \cfallergicchild&\equiv \countfeature{\mathit{allergic}\text{-}\mathit{gluten}[0]\sqcap x_1} \\
    \cfregularchild&\equiv \countfeature{\mathit{not}\text{-}\mathit{allergic}\text{-}\mathit{gluten}[0]\sqcap x_1} \\
    \cfglutenfreesandwichkitchen&\equiv \neg\emptyfeature{\mathit{at}\text{-}\mathit{kitchen}\text{-}\mathit{sandwich}[0]\sqcap x_2} \\
    \cfsandwichkitchen&\equiv \neg\emptyfeature{\mathit{at}\text{-}\mathit{kitchen}\text{-}\mathit{sandwich}[0]\setminus x_2} \\
    \cfsandwichtray&\equiv \neg\emptyfeature{\mathit{ontray}[0]}
\end{align*}

\subsection{Driverlog}

Consider the following concepts and roles:
\begin{align*}
    x_1&\equiv (\mathit{at}_g[0,1]\setminus\mathit{at}[0,1]) \\
    x_2&\equiv (\mathit{at}[1,0]\sqcup\mathit{driving}[1,0])\vert_{\mathit{at}_g[0]\sqcap\mathit{driver}[0]})^{-1} \\
    x_3&\equiv (\mathit{driving}[1]\sqcup(\mathit{at}[1,0]\vert_{\mathit{driver}[0]})[0]) \\
    x_4&\equiv \mathit{at}[0,1]\sqcup\mathit{at}[1,0]\sqcup\mathit{path}[0,1]
\end{align*}
Role $x_1$ is the set of misplaced packages, drivers, and trucks paired with their respective target location.
Role $x_2$ is the set of misplaced drivers paired with their current location.
Concept $x_3$ is the set of trucks with a driver boarded and the location of unboarded drivers.
Role $x_4$ is the set of pairs of trucks and locations reachable by a single unboard, board, or walk action.
The features $\Phi = \{\dfpackages, \dftrucks, \dfdrivergoal, \dfdrivertruck, \dfboarded, \dfloaded \}$
used in the sketch for Driverlog are constructed as follows:
\begin{align*}
    \dfpackages&\equiv \countfeature{\mathit{obj}[0]\sqcap x_1[0]} \\
    \dftrucks&\equiv \countfeature{\mathit{truck}[0]\sqcap x_1[0]} \\
    \dfdrivergoal&\equiv \sumroledistancefeature{x_2}{x_4}{\mathit{at}_g[0,1]} \\
    \dfdrivertruck&\equiv \conceptdistancefeature{x_3}{x_4}{\mathit{truck}[0]\sqcap x_1[0]} \\
    \dfboarded&\equiv \neg\emptyfeature{\mathit{driving}[0]} \\
    \dfloaded&\equiv \neg\emptyfeature{\mathit{obj}[0]\sqcap\mathit{in}[0]\sqcap\mathit{at}_g[0]}
\end{align*}

%

\subsection{Schedule}

The features $\Phi = \{\sfshape, \sfsurface, \sfcolor, \sfhot, \sfoccupied\}$  used in the sketch for Schedule are constructed as follows:
\begin{align*}
    \sfshape&\equiv \countfeature{\mathit{shape}_g[0,1]\setminus\mathit{shape}[0,1]} \\
    \sfsurface&\equiv \countfeature{\mathit{surface}\text{-}\mathit{condition}_g[0,1]\\
    &~~~~~~~~~\setminus\mathit{surface}\text{-}\mathit{condition}[0,1]} \\
    \sfcolor&\equiv \countfeature{\mathit{painted}_g[0,1]\setminus\mathit{painted}[0,1]} \\
    \sfhot&\equiv \countfeature{\mathit{temperature}[0,1]\vert_{\schot}} \\
    \sfoccupied&\equiv \neg\emptyfeature{\mathit{schedule}[0]\sqcup\mathit{busy}[0]}
\end{align*}

\end{subappendices}

\bibliographystyle{kr}
\bibliography{bib/abbrv.bib,bib/literatur.bib,bib/crossref.bib}

\end{document}